\documentclass[a4paper,11pt]{article}

\usepackage{amsmath,amsfonts,amssymb,amsthm}
\usepackage[margin=1in]{geometry}
\usepackage{parskip}
\usepackage{hyperref}
\usepackage{authblk}
\usepackage{bbm}
\usepackage{comment}
\usepackage{accents}
\usepackage{adjustbox}
\usepackage{booktabs}
\usepackage{caption}
\usepackage{graphicx}
\usepackage{subcaption}
\usepackage{multirow}
\usepackage{overpic}
\usepackage{rotating}
\usepackage{algorithm}
\usepackage{algorithmic}

\newcommand{\R}{\mathbb{R}}

\usepackage{tikz}
\usetikzlibrary{shapes,arrows,positioning}

\usepackage{xcolor}
\usepackage[framemethod=tikz]{mdframed}
\usepackage{microtype}
\usepackage{mathtools}
\usepackage{enumitem}
\usepackage[nameinlink,noabbrev,capitalize]{cleveref}
\hypersetup{
    colorlinks,
    citecolor=black,
    filecolor=black,
    linkcolor=black,
    urlcolor=black
}

\usepackage{fancyhdr}
\definecolor{theoremshade}{RGB}{238,244,255}
\definecolor{propshade}{RGB}{255,250,225}
\definecolor{lemmashade}{RGB}{238,250,238}
\definecolor{corshade}{RGB}{245,240,255}
\definecolor{definitionshade}{RGB}{255,243,235}
\definecolor{remarkshade}{RGB}{245,245,245}
\definecolor{exampleshade}{RGB}{238,250,252}
\definecolor{assumptionshade}{RGB}{255,240,240}

\mdfdefinestyle{theoremstyle}{
    backgroundcolor=theoremshade,
    linecolor=theoremshade,
    linewidth=0pt,
    skipabove=10pt,
    skipbelow=10pt,
    innerleftmargin=8pt,
    innerrightmargin=8pt,
    innertopmargin=8pt,
    innerbottommargin=8pt
}
\mdfdefinestyle{propositionstyle}{
    backgroundcolor=propshade,
    linecolor=propshade,
    linewidth=0pt,
    skipabove=10pt,
    skipbelow=10pt,
    innerleftmargin=8pt,
    innerrightmargin=8pt,
    innertopmargin=8pt,
    innerbottommargin=8pt
}
\mdfdefinestyle{lemmastyle}{
    backgroundcolor=lemmashade,
    linecolor=lemmashade,
    linewidth=0pt,
    skipabove=10pt,
    skipbelow=10pt,
    innerleftmargin=8pt,
    innerrightmargin=8pt,
    innertopmargin=8pt,
    innerbottommargin=8pt
}
\mdfdefinestyle{corstyle}{
    backgroundcolor=corshade,
    linecolor=corshade,
    linewidth=0pt,
    skipabove=10pt,
    skipbelow=10pt,
    innerleftmargin=8pt,
    innerrightmargin=8pt,
    innertopmargin=8pt,
    innerbottommargin=8pt
}
\mdfdefinestyle{definitionstyle}{
    backgroundcolor=definitionshade,
    linecolor=definitionshade,
    linewidth=0pt,
    skipabove=10pt,
    skipbelow=10pt,
    innerleftmargin=8pt,
    innerrightmargin=8pt,
    innertopmargin=8pt,
    innerbottommargin=8pt
}
\mdfdefinestyle{examplestyle}{
    backgroundcolor=exampleshade,
    linecolor=exampleshade,
    linewidth=0pt,
    skipabove=10pt,
    skipbelow=10pt,
    innerleftmargin=8pt,
    innerrightmargin=8pt,
    innertopmargin=8pt,
    innerbottommargin=8pt
}
\mdfdefinestyle{assumptionstyle}{
    backgroundcolor=assumptionshade,
    linecolor=assumptionshade,
    linewidth=0pt,
    skipabove=10pt,
    skipbelow=10pt,
    innerleftmargin=8pt,
    innerrightmargin=8pt,
    innertopmargin=8pt,
    innerbottommargin=8pt
}

\theoremstyle{plain}
\newtheorem{theoreminner}{Theorem}[section]
\newtheorem{lemmainner}[theoreminner]{Lemma}
\newtheorem{corollaryinner}[theoreminner]{Corollary}
\newtheorem{propositioninner}[theoreminner]{Proposition}

\theoremstyle{definition}
\newtheorem{definitioninner}[theoreminner]{Definition}
\newtheorem{exampleinner}[theoreminner]{Example}
\newtheorem{assumptioninner}[theoreminner]{Assumption}
\newtheorem{remark}{Remark}[section]

\newenvironment{theorem}[1][]
  {\begin{mdframed}[style=theoremstyle]\begin{theoreminner}[#1]}
  {\end{theoreminner}\end{mdframed}}

\newenvironment{lemma}[1][]
  {\begin{mdframed}[style=lemmastyle]\begin{lemmainner}[#1]}
  {\end{lemmainner}\end{mdframed}}

\newenvironment{corollary}[1][]
  {\begin{mdframed}[style=corstyle]\begin{corollaryinner}[#1]}
  {\end{corollaryinner}\end{mdframed}}

\newenvironment{proposition}[1][]
  {\begin{mdframed}[style=propositionstyle]\begin{propositioninner}[#1]}
  {\end{propositioninner}\end{mdframed}}

\newenvironment{example}[1][]
  {\begin{mdframed}[style=examplestyle]\begin{exampleinner}[#1]}
  {\end{exampleinner}\end{mdframed}}

\title{\bf On the Hidden Biases of  Flow Matching Samplers}

\author{
	Soon Hoe Lim$^{1,2}$\thanks{Corresponding author: \url{shlim@kth.se}.} \vspace{+0.3cm} \\ 
	$^1$Department of Mathematics, KTH Royal Institute of Technology \\
	$^2$Nordita, KTH Royal Institute of Technology and Stockholm University 
}
\date{\today}

\begin{document}
\maketitle

\begin{abstract}
Flow matching (FM) constructs continuous-time ODE samplers by prescribing probability paths
between a base distribution and a target distribution. In this note, we study FM through the lens of finite-sample plug-in estimation. In addition to replacing population expectations by sample
averages, one may replace the target distribution itself by a finite-sample surrogate, ranging
from the empirical measure to a smoothed estimator.
This viewpoint yields a natural hierarchy of empirical FM models. For affine conditional flows, we derive the exact empirical minimizer and identify a smoothed plug-in regime in
which the terminal law is exactly a kernel-mixture estimator. 
This plug-in perspective clarifies several coupled finite-sample biases of empirical FM.
First, replacing the target law by a finite-sample surrogate changes the statistical target.
Second, the empirical minimizer is generally not a gradient field, even when each conditional flow is.
Third, a fixed empirical marginal path does not determine a unique particle dynamics: one may add extra vector fields whose probability flux has zero divergence without changing the marginal path. For Gaussian affine conditional paths, we give explicit families of such flux-null corrections. Finally, the source distribution provides a primary mechanism controlling upper tails of kinetic energy. In particular, Gaussian bases yield exponential upper-tail bounds for instantaneous and integrated kinetic energies, whereas polynomially tailed bases yield corresponding polynomial upper-tail bounds. 

\end{abstract}


\section{Introduction}
The main goal of generative modeling is to use finitely many samples from an unknown target distribution to construct a sampler capable of generating new samples from the same  distribution. Among recent approaches, flow matching (FM) \cite{lipman2022flow, lipman2024flow} and the closely related variants \cite{albergo2023stochastic, liu2022flow} are notable for their flexibility and simplicity, and fit naturally within the broader framework of dynamical measure transport \cite{marzouk2016introduction}.  Given a target probability distribution, FM learns a time-dependent velocity field defining a deterministic continuous transformation that transports a base or source distribution, typically Gaussian, to the target distribution.

A useful way to understand the finite-sample behavior of FM is through the classical distinction between \emph{population} and \emph{plug-in} estimation. In supervised learning \cite{bach2024learning}, one begins with an unknown probability measure $P$ on a measurable space $\mathcal Z$, a pre-specified hypothesis space $\mathcal F$, a loss function $\ell:\mathcal F\times \mathcal Z\to \mathbb R$, and the population risk
\[
R(f):=\int \ell(f,z)\, P(dz), \qquad f\in\mathcal F.
\]
A population risk minimizer is any
\(
f^\star\in \arg\min_{f\in\mathcal F} R(f).
\)
Since $P$ is unknown, one replaces it by a finite-sample surrogate. The most basic choice is the empirical measure
\[
\hat P_n:=\frac1n\sum_{i=1}^n \delta_{Z_i},
\qquad Z_1,\dots,Z_n \overset{\mathrm{i.i.d.}}{\sim} P,
\]
which yields the empirical risk
\[
\hat R_n(f):=\int \ell(f,z)\,\hat P_n(dz)=\frac1n\sum_{i=1}^n \ell(f,Z_i).
\]
This is empirical risk minimization. A second possibility is to replace $P$ by a regularized plug-in estimator $\tilde P_{n,h}$, for instance one induced by a kernel density estimator, and to work instead with
\[
\hat R_{n,h}(f):=\int \ell(f,z)\,\tilde P_{n,h}(dz).
\]

This classical picture already suggests three distinct regimes:
\begin{enumerate}[label=(\roman*),leftmargin=2.5em]
    \item \textbf{Population level.} One reasons directly with the unknown law $P$.
    \item \textbf{Raw empirical plug-in.} One replaces $P$ by the empirical measure $\hat P_n$.
    \item \textbf{Smoothed plug-in.} One replaces $P$ by a regularized estimator $\tilde P_{n,h}$.
\end{enumerate}
The third regime is fundamental in nonparametric statistics \cite{gyorfi2002distribution, tsybakov2008nonparametric}: smoothing introduces a bias--variance tradeoff and, in ambient dimension $d$, inherits the familiar curse of dimensionality of kernel-based estimation.

The same hierarchy appears naturally in FM, but now the unknown object is not only a risk functional but an entire target law. Let $p_0$ be a base distribution on $\mathbb R^d$ and let $p_1$ be the unknown target distribution\footnote{We use the little $p$ notation for the probability distributions appearing in FM, not to be confused with the broader statistical introduction earlier, where we use big $P$. We use \(n\) for the generic statistical discussion and \(N\) for the FM training sample size below. }. At the population level, one studies a velocity field that transports $p_0$ to $p_1$. At the first finite-sample level, one keeps $p_1$ fixed but replaces expectations in the FM objective by Monte Carlo averages. At the second level, one replaces the target law $p_1$ itself by the empirical measure
\[
\hat p_1:=\frac1N \sum_{i=1}^N \delta_{x^{(i)}},
\qquad x^{(1)},\dots,x^{(N)} \overset{\mathrm{i.i.d.}}{\sim} p_1,
\]
leading to a \emph{raw empirical} FM model. At the third level, one replaces \(p_1\) by a smoothed plug-in estimator \(\tilde p_{1,h}\) (e.g., a kernel density estimator). These replacements are mathematically distinct, and they induce different structural biases in the resulting sampler.

Our starting point is that these three regimes should not be conflated. At the population level, some FM constructions admit gradient-field velocities, a property shared by Benamou--Brenier optimal flows, though not sufficient by itself for optimality. By contrast, the exact raw empirical minimizer is a spatially weighted mixture of conditional velocity fields. Consequently, even when each conditional velocity field is itself a gradient field, the empirical minimizer typically is not. Thus the finite-sample plug-in geometry of FM differs in an essential way from its population counterpart. On the other hand, the smoothed plug-in viewpoint reveals a natural intermediate regime: for affine conditional flows with positive terminal scale, averaging conditional terminal laws over the empirical target measure gives exactly a kernel density estimator. This connects empirical flow matching (EFM) directly to classical nonparametric smoothing, together with its attendant bias--variance tradeoff and high-dimensional limitations.

One of the goals of this note is to make this picture precise. We begin with a brief statistical prelude on population, empirical, and smoothed plug-in estimation. We then review FM and conditional flow matching (CFM), derive the exact empirical minimizer for affine conditional flows, identify conditions under which this minimizer fails to be a gradient field, isolate the smoothed plug-in regime inside affine flows, and analyze the kinetic energy of the resulting samplers. We also introduce an equivalence relation  on empirical samplers: two velocities are equivalent if they induce the same divergence of probability flux \cite{horvat2024gauge} against the empirical marginal. This separates the density path from the particle dynamics realizing it. Taken together, these results show that finite-sample FM modifies the statistical target, the transport geometry, the particle-level dynamics, and the energetic behavior of the learned sampler in a coupled way. We complement the theoretical analysis with numerical experiments on exact empirical affine-flow samplers, showing that Gaussian bases produce light kinetic energy tails while Student-\(t\) bases produce notably heavier energy profiles, in agreement with the source-tail mechanism suggested by the theory.

Our main contributions are as follows.
\begin{itemize}[leftmargin=2em]
    \item We formulate and study a plug-in hierarchy for finite-sample flow matching, distinguishing objective-level empirical approximation from empirical target and smoothed empirical target plug-in models.
    \item For affine conditional flows, we derive the exact empirical minimizer and show that positive terminal scale yields a kernel density estimator at terminal time. We further prove that the raw empirical minimizer is generally not a gradient field, even when the individual conditional velocity fields are gradients, thereby identifying a finite-sample geometric obstruction to Benamou--Brenier optimality.
    \item We show that EFM samplers are not uniquely determined by their marginal
    density paths: different velocity fields can generate the same empirical density
    evolution while inducing different particle trajectories. For variance-floored
    rectified flow and Gaussian affine conditional paths, we construct explicit
    families of such equivalent samplers.
    \item We identify the source distribution as a key driver of kinetic energy
tails in EFM samplers. Gaussian sources produce light energy tails, whereas
polynomially tailed sources can produce substantially heavier ones. We prove
corresponding upper-tail bounds, show their stability under controlled
marginal-preserving velocity modifications, and explain why such growth control
is necessary. 
\end{itemize}
We further illustrate these mechanisms with toy numerical experiments.

While several ingredients used below are classical, the contribution of this note is to assemble them into a finite-sample plug-in analysis of FM and CFM. This viewpoint reveals that empirical target replacement simultaneously changes the terminal statistical target, destroys gradient structure in the empirical minimizer, leaves particle dynamics non-unique at fixed marginal path, and imposes source-dependent kinetic energy upper-tail behavior.

Throughout, we use the common shorthand of writing \(p\) both for a probability law and,
when it exists, its density with respect to Lebesgue measure. Thus expressions such as
\(X\sim p\), \(T_{\#}p_0=p_1\), and \(p_t(z)\) should be interpreted according to context:
\(p\) denotes a probability measure in sampling and pushforward statements, and a density
when evaluated at a point or integrated against Lebesgue measure. Empirical distributions are
denoted by \(\hat p\) and are understood as probability measures, not densities.
Proofs of theoretical results are deferred to the appendix.


\section{A Statistical Prelude: Population and Plug-In Estimation}
\label{sec:plugin}

Before turning to FM, it is useful to isolate the statistical template that underlies our finite-sample viewpoint. Let $(\mathcal Z,\mathcal A)$ be a measurable space, $P$ be an unknown probability measure on $\mathcal Z$, $Z_1,\dots,Z_n \overset{\mathrm{i.i.d.}}{\sim} P$, $\mathcal F$ be a hypothesis space, and $\ell:\mathcal F\times\mathcal Z\to\mathbb R$ be a measurable loss function. The population risk is
\[
R(f):=\int \ell(f,z)\,P(dz), \qquad f\in\mathcal F.
\]
Any element of
$\arg\min_{f\in\mathcal F} R(f)
$
will be called a \emph{population risk minimizer}.

Since $P$ is unknown, one cannot evaluate $R$ directly. The empirical plug-in principle replaces $P$ by the empirical measure
\[
\hat P_n:=\frac1n\sum_{i=1}^n \delta_{Z_i},
\]
which leads to the empirical risk
\[
\hat R_n(f):=\int \ell(f,z)\,\hat P_n(dz)
=\frac1n\sum_{i=1}^n \ell(f,Z_i).
\]
Any minimizer of $\hat R_n$ is the empirical risk minimization estimator.

A more regularized alternative is to replace $P$ by a smoothed estimator. When $\mathcal Z=\mathbb R^d$ and $P$ is absolutely continuous with density $p$, a standard choice is the kernel estimator
\[
\tilde p_{n,h}(z):=\frac1n\sum_{i=1}^n K_h(z-Z_i),
\qquad
K_h(z):=h^{-d}K(z/h),
\]
where $K:\mathbb R^d\to[0,\infty)$ is a kernel with $\int_{\mathbb R^d} K(z)\,dz=1$, and $\tilde P_{n,h}$ denotes the probability measure with density $\tilde p_{n,h}$. One then studies the smoothed plug-in risk
\[
\hat R_{n,h}(f):=\int \ell(f,z)\,\tilde P_{n,h}(dz).
\]

To quantify the effect of replacing $P$ by another probability measure $Q$, it is convenient to use the total variation norm of a finite signed measure\footnote{We use the signed-measure convention for total variation, so for probability measures this equals twice the usual total variation distance used in probability theory.} $\mu$,
\(
\|\mu\|_{\mathrm{TV}}
:=\sup_{\|g\|_\infty\le 1}\left|\int g\,d\mu\right|.
\)
If $|\ell(f,z)|\le M$ uniformly in $(f,z)$, then for every probability measure $Q$ on $\mathcal Z$,
\[
\left| \int \ell(f,z)\,(P-Q)(dz)\right|
\le M\,\|P-Q\|_{\mathrm{TV}}.
\]
In particular,
$|R(f)-\hat R_{n,h}(f)|
\le M\,\|P-\tilde P_{n,h}\|_{\mathrm{TV}}.$
Thus, control of the plug-in approximation at the level of measures directly yields control of the induced error in the risk functional.

The distinction between population, empirical, and smoothed plug-in estimation is classical, but it is especially useful for our purposes because an analogous trichotomy appears in FM. There, the unknown object is no longer only a risk functional but an entire target law. One may either approximate expectations under that law by Monte Carlo averages, replace the target law by the empirical measure itself, or replace it by a smoothed surrogate. The remainder of this note shows that these choices lead to genuinely different FM models, with different geometric and statistical consequences.

\section{Flow Matching (FM) and Conditional Flow Matching (CFM)} \label{sec:background}
Let $p_0$ and $p_1$ be source and target probability measures on $\mathbb R^d$, with densities denoted with the same symbol $p_0$ and $p_1$ when they exist. For instance, \(p_1\) may be the data distribution \(p^*\), or a smoothed version of it. We say that \(T\) is a transport map if \(Z\sim p_0\) implies \(T(Z)\sim p_1\), in which case we write \(T_{\#}p_0=p_1\). 
A common generative modeling paradigm aims to learn such a transport map using samples \(x^{(i)}\sim p_1\), where \(p_1\) is typically unknown  \cite{peyre2025optimal}. One popular approach under this paradigm is flow matching (FM).

\noindent {\bf FM.} The goal of FM is to find a velocity field $v: [0,1] \times  \R^d \to \R^d$, such that, if we solve the ODE:
$$\frac{dz(t)}{dt} = v(t, z(t)), \  z(0) = z_0 \in \R^d,$$ then the law of $z(1)$ when $z_0 \sim p_0$ is $p_1$ (in which case we say that $v$ drives $p_0$ to $p_1$). The law of $z(t)$ for $t \in [0,1]$ is described by  a probability path $p: [0,1] \times \R^d \to \R$, denoted $p_t(z)$, that evolves from $p_0$ at $t=0$ to $p_1$ at $t=1$. If we know $v$, then we can first sample $z_0 \sim p_0$ and then evolve the ODE from $t=0$ to $t=1$ to generate new samples.

The velocity field $v$ generates the flow $\psi: [0,1] \times \R^d \to \R^d$ given as $\psi_t(z) = z(t)$, and the probability path via the push-forward distributions: $p_t = [\psi_t]_{\#} p_0$, i.e., $\psi_t(Z) \sim p_t$ for $Z \sim p_0$. In particular, $Z \sim p_0$ implies that $\psi_1(Z) \sim p_1$, i.e., $\psi_t$ can be viewed as a dynamical transport map. The ODE corresponds to the Lagrangian description (the $v$-generated trajectories viewpoint), and a change of variable links it to the Eulerian description (the evolving probability path $p_t$ viewpoint). Indeed, under suitable regularity and integrability assumptions \cite{wald2025flow, albergo2024learning, albergo2023stochastic}, a flow generated by \(v\) induces a density path satisfying the continuity equation
\begin{equation} \label{eq_continuity}
    \frac{\partial p_t}{\partial t} + \nabla \cdot (p_t v) = 0,
\end{equation}
where $\nabla \cdot$ denotes the divergence operator. 
Conversely, sufficiently regular solutions of the continuity equation can be represented by flows solving the ODE.
This equation ensures that the flow defined by $v$ conserves the mass (or probability) described by $p_t$.  In general, even for simple prescribed probability paths between \(p_0\) and \(p_1\), the velocity field does not admit a closed-form expression when  $p_0$ and $p_1$ are known, except in special cases such as Gaussians, mixture of Gaussians and uniform distributions \cite{mena2025statistical}.

The above description gives us a population FM model, which we aim to learn using a finite number of samples in practice. Given such a $v$, it is standard to learn it with a parametric model $v_\theta$ (e.g., neural network) by minimizing the FM objective:
\begin{equation} 
L_{\text{FM}}[v_\theta] = \mathbb{E}_{t \sim \mathcal{U}[0,1], \ Z_t \sim p_t}[\|v_\theta(t, Z_t) - v(t, Z_t)\|^2]. 
\end{equation}

\noindent {\bf CFM.} In CFM \cite{lipman2022flow, tong2023improving}, we consider a  probability path in the mixture form:
\begin{equation} \label{mix_probpath}
p_t(z) = \int p_t(z |  x)\,p_1(dx), 
\end{equation}
where $p_t(\cdot | x): \R^d \to \R^+$ is a conditional probability path generated by some vector field $v(t, \cdot | x):  \R^d \to \R^d$ for $x \in \R^d$. Moreover, consider the vector field:
\begin{equation}
v(t,z)
=
\int v(t,z | x)
\frac{p_t(z | x)}{p_t(z)}
\,p_1(dx),    
\end{equation}
assuming $p_t(z)>0$.
In this setting, it can be shown in \cite{lipman2022flow} that minimizing the FM objective $L_{\text{FM}}$ is equivalent to minimizing the CFM objective:
\begin{equation}
    L_{\text{CFM}}[v_\theta]  = \mathbb{E}_{t \sim \mathcal{U}[0,1], \ X \sim p_1, \ Z_t \sim p_t(\cdot | X)}[\|v_\theta(t, Z_t) - v(t, Z_t| X)\|^2].
\end{equation}

In order to apply CFM, we need to specify the boundary distributions $p_0$ and $p_1$, and the conditional probability path $p_t(z|x)$. Below are some examples.

\begin{example}[Rectified Flow] \label{ex_RF}
    A canonical choice \cite{liu2022flow} is $p_0 = \mathcal{N}(0, I_d)$, $p_1 = p^*$, and 
    \begin{equation} \label{eq_probpath_RF}
         p_t(z | X = x_1) = \mathcal{N}(z; t x_1, (1-t)^2 I_d),
    \end{equation}
    which corresponds to the conditional velocity field $v(t, z| X=x_1) = \frac{x_1 - z}{1-t}$. This conditional probability path realizes linear interpolating paths of the form $Z_t = (1-t) x_0 + t x_1$ between a (reference) Gaussian sample $x_0$ and a data sample $x_1$. In practice,  regularized versions of rectified flow are preferred for numerical stability  (since $v$ blows up as $t \to 1$). A simple version is to modify the conditional probability path to $$p_t(\cdot | X = x_1) = \mathcal{N}(t x_1, (1-(1-\sigma_{min})t)^2 I_d),$$ for some small $\sigma_{min} > 0$, which corresponds to the regularized conditional velocity field $v(t, z| X=x_1) = \frac{x_1 - (1-\sigma_{min})z}{1-(1-\sigma_{min})t}$. Another version is to  consider a smoothed version of the data distribution $p^*$; e.g., $p_1 = p^* \star \mathcal{N}(0, \sigma_{min}^2 I_d)$, where $\star$ denotes convolution. Variance flooring modifies the conditional path, whereas replacing \(p^\ast\) by \(p^\ast\ast N(0,\sigma_{\min}^2I_d)\) changes the terminal target law.
\end{example}

\begin{example}[Affine  Flows] \label{ex_gen}
More generally, consider a latent variable $Z \sim \mathbb{Q}$ with positive probability density function (PDF) $K > 0$ (not necessarily Gaussian) and, for $t \in [0,1]$,  the affine conditional flow defined by $\psi_t(Z|X) = m_t(X) + \sigma_t(X) Z$ for some time-differentiable functions $m: [0,1] \times \R^d \to \R^d$ and $\sigma: [0,1] \times \R^d \to \R^+$. Since $\psi_t$ is linear in $Z$, we can obtain its density via the change of variables:
\begin{equation} \label{eq_probpath_affine}
    p_t(z|X) = \frac{1}{\sigma_t^d(X)} K\left(\frac{z-m_t(X)}{\sigma_t(X)} \right).
\end{equation}
Here \(\sigma_t(X)\) is a positive scalar scale. Matrix-valued affine maps would require a matrix-valued coefficient \(a_t\) and are not considered here.

Then, as in Theorem 3 in \cite{lipman2022flow}, we can show that the unique vector field that defines $\psi_t(\cdot | X)$ via the ODE $\frac{d}{dt} \psi_t(z|X) = v(t, \psi_t(z|X) | X)$ has the form: 
\begin{equation}
    v(t, z| X) = a_t(X) z + b_t(X),
\end{equation}
where 
\begin{align} \label{eq_ab}
    a_t(X) &= \frac{\frac{\partial \sigma_t}{\partial t}(X)}{\sigma_t(X)}, \quad 
    b_t(X) = \frac{\partial m_t}{\partial t}(X) - m_t(X) a_t(X).
\end{align}
This family of flows is also studied in \cite{kunkel2025minimax}. The rectified flow in the previous example is a special case of this family of conditional flows (with $K = \mathcal{N}(0, I_d)$,  $m_t(X) = tX$ and $\sigma_t(X) = 1-t$). The Gaussian flows considered in \cite{lipman2022flow, tong2023improving, albergo2023stochastic} are also special cases. 
\end{example}

All the formulations thus far are in the idealized continuous-time setting. In practice, we work with  Monte Carlo estimates of the objective and use the optimized $v_\theta$ to generate new samples by simulating the ODE with a numerical scheme. Note, however, that the training of CFM is simulation-free: the dynamics are only simulated at inference time and not when training the parametric (neural network) model. In practice, affine flows are most widely used, and thus we will focus on them here, using the rectified flow model as a canonical example.

\section{Empirical and Smoothed Plug-in Flow Matching}  \label{sec:empiricalfm}

Suppose that we are given a source distribution $p_0$ and $N$ i.i.d. samples $x^{(1)},\dots,x^{(N)} \sim p_1$, so that the target law is observed only through finite data. At this point it is useful to distinguish three levels of approximation.

\begin{enumerate}[label=(\roman*),leftmargin=2.5em]
    \item \textbf{Objective-level empirical plug-in.} One keeps the target law $p_1$ conceptually fixed, but replaces expectations appearing in $L_{\mathrm{FM}}$ or $L_{\mathrm{CFM}}$ by Monte Carlo averages.
    \item \textbf{Raw empirical target plug-in.} One replaces the target law itself by the empirical distribution
    \[
    \hat p_1 := \frac1N \sum_{i=1}^N \delta_{x^{(i)}}.
    \]
    This is the most singular finite-sample surrogate of $p_1$.
    \item \textbf{Smoothed empirical target plug-in.} One instead uses a regularized estimator $\tilde p_{1,h}$, for example a kernel density estimator. This is the natural nonparametric counterpart of replacing the empirical measure by a smoothed plug-in estimator in classical statistics.
\end{enumerate}

We shall begin our study with the raw empirical target plug-in, since it leads to closed-form expressions and exposes the main geometric bias. We then explain how the same formalism naturally produces smoothed plug-in targets.

When $p_1$ is replaced by the empirical measure $\hat p_1$, the empirical counterparts of $p_t(z)$ and $v(t,z)$ are given by
\begin{align}
    \hat{p}_t(z) &= \frac{1}{N} \sum_{i=1}^N p_t(z|x^{(i)}), \\
    \hat{v}(t,z) &= \sum_{i=1}^N v(t,z| x^{(i)}) \frac{p_t(z|x^{(i)})}{\sum_{j=1}^N p_t(z|x^{(j)})}   
\end{align}
respectively. The objectives that the empirical FM and empirical CFM minimize are then given by, respectively:
\begin{align}
    \hat{L}_{\text{FM}}[v'] &= \mathbb{E}_{t \sim \mathcal{U}[0,1],  \ Z_t \sim \hat{p}_t}[\|v'(t, Z_t) - \hat{v}(t, Z_t)\|^2], \\ 
    \hat{L}_{\text{CFM}}[v'] &= \mathbb{E}_{t \sim \mathcal{U}[0,1], \ X \sim \hat{p}_1, \ Z_t \sim p_t(\cdot | X)}[\|v'(t, Z_t) - v(t, Z_t| X)\|^2] \nonumber \\
    &= \frac{1}{N} \sum_{i=1}^N \mathbb{E}_{t \sim \mathcal{U}[0,1], \ Z_t \sim p_t(\cdot | x^{(i)})}[\|v'(t, Z_t) - v(t, Z_t| x^{(i)})\|^2],
\end{align}
where $p_t(\cdot|x^{(i)})$ is the conditional probability path (given by, e.g., \eqref{eq_probpath_affine} or \eqref{eq_probpath_RF}).

One can show that if $v(t, \cdot | x^{(i)})$ generates $p_t(\cdot | x^{(i)})$ for all $i \in [N]$, then $\hat{v}(t, \cdot)$ generates $\hat{p}_t$ (see Lemma 2.1 in \cite{kunkel2025minimax}).
Just as before, the equivalence (with respect to the optimizing arguments) between FM and CFM carries over to empirical FM and empirical CFM naturally (see Theorem 2.2 in \cite{kunkel2025minimax}). Moreover, over an unrestricted square-integrable function class, the examples of conditional probability paths considered earlier admit a closed-form minimizer $\hat{v}^* \in \text{argmin}_v \hat{L}_{CFM}[v] = \text{argmin}_v \hat{L}_{FM}[v]$, giving us a training-free model for generating new samples. This sampler is described by the ODE: 
\begin{equation} \label{eq_closedformsampler}
    \frac{d \hat{z}^*(t)}{dt} = \hat{v}^*(t, \hat{z}^*(t)), \quad \hat{z}^*(0) \sim p_0,
\end{equation}
which we evolve to terminal time in regularized cases, or to \(T<1\) in singular unregularized cases.

\begin{example}[Empirical Rectified Flow] For the rectified flow example in Example \ref{ex_RF}, the minimizer $\hat{v}^*$ has a closed-form formula (see \cite{bertrand2025closed} for derivation):
\begin{equation} \label{eq_empRF}
    \hat{v}^*(t, z) = \sum_{i=1}^N w_i(t, z) \frac{x^{(i)}-z}{1-t},
\end{equation}
where 
\[
w_i(t,z)
=
\frac{\exp\!\left(-\|z-tx^{(i)}\|^2/[2(1-t)^2]\right)}
{\sum_{j=1}^N \exp\!\left(-\|z-tx^{(j)}\|^2/[2(1-t)^2]\right)},
\]
or equivalently,
$w_i(t, z) = \text{softmax}_i\left(\left(-\frac{1}{2(1-t)^2}\|z- t x^{(j)}\|^2  \right)_{j \in [N]} \right)$, with  $\text{softmax}_i$ denoting the $i$th component of the vector obtained after applying the softmax operation. This empirical minimizer is thus a time-dependent weighted average of the $N$ different directions towards the $x^{(i)}$. Similar formula can also be obtained for regularized versions of rectified flow.
\end{example}

\begin{example}[Empirical Affine Flows and Smoothed Plug-in Targets]  \label{eg_empaffineflows}
The affine family also exhibits the smoothed plug-in regime in a particularly transparent way. Fix a  PDF $K > 0$ on $\R^d$, take $p_0(z)=K(z)$, and choose any $m_t$ and $\sigma_t$ such that
\[
m_0(X)=0, \qquad m_1(X)=X, \qquad \sigma_0(X)=1, \qquad \sigma_1(X)=\sigma_{\min}>0.
\]
Then the terminal conditional density is
\[
p_1(z\mid X=x^{(i)}) = \frac{1}{\sigma_{\min}^d}K\!\left(\frac{z-x^{(i)}}{\sigma_{\min}}\right),
\]
and averaging over the empirical target law yields the terminal marginal
\[
\tilde p_1(z)=\frac1N\sum_{i=1}^N p_1(z\mid X=x^{(i)})
=\frac{1}{N \sigma_{\min}^d} \sum_{i=1}^N K\!\left(\frac{z-x^{(i)}}{\sigma_{\min}}\right).
\]
Thus the terminal law is exactly the equally weighted kernel density estimator associated with kernel $K$ and bandwidth $\sigma_{\min}$. In particular, the affine-flow construction already contains a smoothed plug-in estimator of the target distribution. If $K$ is the standard Gaussian density, then this family converges formally to the rectified flow regime as the terminal bandwidth $\sigma_{\min} \downarrow 0$.
\end{example}

Moreover, similar to the empirical rectified flow, we can obtain a closed-form formula for the raw empirical target affine-flow minimizer.

\begin{proposition} \label{prop1} 
For the family of affine flows in Example \ref{eg_empaffineflows}, the minimizer of the empirical FM objective over
\(L^2(dt\,\hat p_t(dz);\mathbb R^d)\) is unique \(dt\otimes \hat p_t\)-a.e. and, for a.e. \(t\), is given \(\hat p_t\)-a.e. by the closed-form formula:
    \begin{equation}
        \hat{v}^*(t, z) = \sum_{i=1}^N w_i(t,z) \cdot (a_t(x^{(i)}) z + b_t(x^{(i)})),
    \end{equation}
where $a_t$ and $b_t$ are given in \eqref{eq_ab}, and $w_i(t,z)$ is the kernel-dependent weighting function
\begin{equation}
    w_i(t,z) = \frac{p_t(z|x^{(i)})}{\sum_{j=1}^N p_t(z|x^{(j)})},
\end{equation}
with 
\begin{equation}
    p_t(z|x^{(i)}) = \frac{1}{\sigma_t^d(x^{(i)})} K\left(\frac{z-m_t(x^{(i)})}{\sigma_t(x^{(i)})} \right).
\end{equation}
\end{proposition}

Intuitively, $\hat{v}^*$ is a convex combination of the individual conditional velocity fields $v(t, z| x^{(i)})$, weighted by $w_i(t, z)$, where $w_i(t,z)$ represents the posterior responsibility that the point $z$ at time $t$ originated from the $i$th conditional path.

\section{Structural and Energetic Biases of EFM Samplers} \label{sec:energy}

We now analyze the geometric and energetic consequences of the raw empirical target plug-in. The first issue is structural: does the exact empirical minimizer retain the gradient-field property associated with optimal transport (OT) in some population models? The second issue is energetic: regardless of optimality, what can be said about the kinetic energy of the resulting trajectories?

\subsection{Background}
We begin by recalling the OT benchmark with which these questions are naturally aligned.

\noindent {\bf Optimal Transport.} OT is the problem of efficiently moving probability mass from a source distribution $p_0$ to a target distribution $p_1$ such that a given cost function has minimal expected value. More precisely, we aim to find a coupling $(Z_0, Z_1)$ of random variables $Z_0 \sim p_0$ and $Z_1 \sim p_1$ such that the expected cost $\mathbb{E}[c(Z_0, Z_1)]$ is minimal, where $c$ is a cost function, typically chosen as $c_1(z_0, z_1) := \|z_0 - z_1\|$  or $c_2(z_0, z_1) := \|z_0 - z_1\|^2$ \cite{chewi2024statistical, peyre2025optimal}. 

The Monge map (or OT map) $T_0$ is the transport map that minimizes $\mathbb{E}_{p_0}[c_2(Z_0, T(Z_0))]$. The squared 2-Wasserstein distance $W_2^2(p_0, p_1)$ is defined by the minimum expected squared distance over all couplings:$$W_2^2(p_0, p_1) := \inf_{\gamma \in \Pi(p_0, p_1)} \mathbb{E}_{ (Z_0, Z_1) \sim \gamma } [\|Z_0 - Z_1\|^2] = \inf_{\gamma \in \Pi(p_0, p_1)} \int \|x - y\|^2 d\gamma(x,y),$$ 
where $\Pi(p_0, p_1)$ is the set of all joint probability distributions with marginals $p_0$ and $p_1$. 
Under suitable conditions, for instance when \(p_0\) is absolutely continuous and \(p_0,p_1\in\mathcal P_2(\mathbb R^d)\), this minimum is achieved by a Monge map \(T_0\), such that $W_2^2(p_0, p_1) = \mathbb{E}_{Z_0 \sim p_0}[\|Z_0 - T_0(Z_0)\|^2]$.  The Wasserstein distance $W_2$ defines a metric on $\mathcal{P}_2(\R^d)$, the space of probability measures on $\R^d$ with finite second moment.

If \(p_0\) is absolutely continuous and \(p_0,p_1\in\mathcal P_2(\mathbb R^d)\), then Brenier's theorem gives a unique \(p_0\)-a.e. optimal map \(T_0=\nabla\Phi\) for a convex function \(\Phi\). More precisely, 
let $\mathcal{T}(p_0, p_1) := \{T: \R^d \to \R^d : T_{\#} p_0 = p_1\}$. The following is a key result in OT theory due to Brenier (see, e.g.,  Chapter 3 in \cite{villani2021topics}, \cite{manole2024plugin}): there exists a unique (up to a $p_0$-negligible set)  minimizer $T_0$ to the Monge problem:
    $$d_{\mathrm{Monge}}(p_0, p_1)^2 := \inf_{T \in \mathcal{T}(p_0, p_1)} \int \|x - T(x)\|^2 dp_0(x) $$
such that $d_{\mathrm{Monge}}(p_0, p_1)^2 = W_2^2(p_0, p_1)$. Moreover, $T_0$ can be represented ($p_0$-almost everywhere) as $T_0 = \nabla \Phi$ for some convex function $\Phi: \R^d \to \R$ (this $T_0$ is  the OT map).

\noindent {\bf Dynamical Representation (Benamou-Brenier Formulation).} Like any sufficiently regular transport map, OT map can be expressed in a dynamic form as a continuous flow from the source distribution $p_0$ to the target distribution $p_1$ \cite{benamou2000computational, chen2021stochastic}. Consider a flow $\psi_t(z)$ defined by the ODE:
$$\frac{\partial}{\partial t} \psi_t(z) = v(t, \psi_t(z)), \quad \text{for all } t \in [0, 1],$$
for a velocity field $v(t, z)$, with the initial condition $\psi_0(z) = z$. The flow $\psi_t$  induces a probability path, $p_t = [\psi_t]_{\#} p_0$, in the Wasserstein space \cite{wald2025flow}.

Let $\mathcal{U}$ be the collection of all velocity fields $v$ such that the flow $\psi_t(z)$ is uniquely defined and transports $p_0$ to $p_1$ over the unit time interval. The OT  map $T_0(z)$ is given by the end-point of the optimal flow: $T_0(z) = \psi_1^{\text{OT}}(z)$, where the associated optimal velocity field $v^{\text{OT}}(\cdot, \cdot)$ is the minimizer of the expected kinetic energy\footnote{This is also, up to a multiplicative constant involving $d$, the kinetic energy considered in \cite{shaul2023kinetic}.}:
$$\mathbb{E} \left[ \int_0^1 \|v(t, \psi_t(Z_0))\|^2 dt \right]$$
over all $v \in \mathcal{U}$. This minimal expected energy is equal to the squared 2-Wasserstein distance $W_2^2(p_0, p_1)$. Importantly,  the $W_2$ optimal velocity field $v^{\text{OT}}$ must be irrotational (curl-free), meaning that $v^{\text{OT}}(t, z) = -\nabla_z  \Phi(t, z)$ for some scalar potential $\Phi$ (otherwise, intuitively the curl component would introduce unnecessary looping or rotational motion, which would increase the total cost); see also Theorem 8.3.1 in \cite{ambrosio2005gradient}. 

If $p_t$ denotes the density of the distribution at time $t$ (i.e., the law of $\psi_t(Z_0)$), the optimal solution must satisfy the continuity equation (which ensures mass conservation):
$$\partial_t p_t + \nabla \cdot (v^{\text{OT}} p_t) = 0.$$
Hence, the optimization problem (Benamou-Brenier formulation) can be written in its Eulerian form, and minimizes the total kinetic energy over all admissible paths:
$$\inf_{v, p} \int_0^1 \int_{\mathbb{R}^d}  \|v(t, z)\|^2 p_t(z) \, dz \, dt$$$$\text{subject to } \quad \partial_t p_t + \nabla \cdot (v_t p_t) = 0,$$
with the boundary conditions  $p_0$ (at $t=0$) and  $p_1$ (at $t=1$). 

\noindent {\bf Empirical Continuity Equation.} 
Now, the empirical counterpart of the continuity equation \eqref{eq_continuity} is:
\begin{equation} \label{eq_emp_continuity}
    \frac{\partial \hat{p}_t}{\partial t} + \nabla \cdot (\hat{p}_t \, \hat{v}(t,\cdot)) = 0.
\end{equation}
In particular, the empirical minimizer satisfies $\hat v^{*}(t,z)=\hat v(t,z)$ pointwise, and hence the pair $(\hat p_t,\hat v^{*})$ also satisfies \eqref{eq_emp_continuity}.


It is natural to ask if the $\hat{v}^*$ (the velocity field that a trainable CFM model is really optimizing for) in \eqref{eq_empRF} and Proposition \ref{prop1} corresponds to an optimal velocity field in the OT sense. 
In fact, except for special cases, even the velocity fields $v_t$ arising from the population FM framework are generally not gradient functions \cite{wald2025flow, liu2022rectified}, thus not optimal in the OT sense. Indeed, OT paths are generally outside the class of probability paths with affine conditionals. Since affine conditionals are of particular interest due to the fact that they enable scalable training, \cite{shaul2023kinetic} studied the kinetic optimal path within this class of paths using a proxy for the kinetic energy.

The following example gives a special case in which we have velocity fields which can be represented as gradient fields. We will look at the empirical case later.

\begin{example}[The Population RF Regression Minimizer Can Be a Gradient Field] \label{ex_gf}
If the joint distribution of the source and target is a product distribution, i.e., $p_{0,1} = p_0 \times p_1$ (independent coupling), then for the interpolating path of the rectified flow $Z_t = (1-t) x_0 + tx_1$, $x_0 \sim \mathcal{N}(0,I_d)$, and $p_1 \in \mathcal{P}_2(\R^d)$, the population regression minimizer can be shown to be the conditional expectation \cite{wald2025flow, zhang2024flow}:
\begin{equation}
    v(t,z) = \mathbb{E}_{x_0 \sim p_0, \ x_1 \sim p_1}[x_1 - x_0 | Z_t = z] =  \nabla_z \Phi(t, z), 
\end{equation}
where 
\begin{equation} \label{eq_Phi}
    \Phi(t,z) = \frac{1}{2t} \|z\|^2 + \frac{1-t}{t} \log p_t(z),
\end{equation}
for $t \in (0,1)$.
We also see that the score function is related to the velocity by: $\nabla_z \log p_t(z) = \frac{t}{1-t} v(t,z) - \frac{1}{1-t} z.$ An analogous formula can also be derived for a more general flow  with $Z_t = \alpha_t x_0 + \beta_t x_1$ for some time-differentiable $\alpha_t$, $\beta_t$ such that $\alpha_0 = \beta_1 = 1$ and $\alpha_1 = \beta_0 = 0$. This tells us that the rectified flow's regression minimizer, under the independent coupling, is a gradient field (but does not generally give us an OT map due to the independent coupling assumption; being a gradient field is necessary but not sufficient for OT).
\end{example}

Let us consider Gaussian distributions for $p_0$ and $p_1$, in which case  the OT map can be computed explicitly \cite{delon2020wasserstein}.

\begin{example}[Explicit Examples; See  \cite{mena2025statistical}] \label{ex_gaussianRF}
Take $p_0 = \mathcal{N}(0, \Sigma_0)$, $p_1 = \mathcal{N}(m_1, \Sigma_1)$ and consider the rectified flow (RF) map, denoted $R(x) := x + \int_0^1 v(t, \psi_t(x)) dt$ with $v = \dot{\psi}_t$, where $\psi_t(x) = (1-t) x + t R(x)$ is the displacement interpolation between the independent Gaussians $X_0 \sim p_0$ and $X_1 \sim p_1$. 

If $\Sigma_0 = I_d$, then  Monge's OT map and the RF map between $X_0$ and $X_1$  coincide: $T_0(x) = m_1 + \Sigma_1^{1/2} x = R(x)$. In this Gaussian case, the population RF map can be computed explicitly. However, if $\Sigma_0 \neq I_d$, then the two maps are not equivalent. 
\end{example}

\noindent {\bf Raw empirical target plug-in generally destroys gradient structure.}
A crucial observation is that even if the relevant population velocity is a gradient field, the exact raw empirical target plug-in minimizer is generally not. The obstruction is entirely due to the spatially varying posterior weights $w_i(t,z)$ appearing in Proposition \ref{prop1}. This is the main content of the following proposition.

    \begin{proposition} \label{prop2}
    Assume \(d\ge2\).
    Let the empirical target distribution be $\hat{p}_1 = \frac{1}{N} \sum_{i=1}^N \delta_{x^{(i)}}$. Consider the family of empirical affine flows defined by the conditional probability paths $p_t(z|x^{(i)})$ and their corresponding conditional velocity fields $v_i(t,z) := v(t, z| x^{(i)}) = a_t(x^{(i)}) z + b_t(x^{(i)})$ from Proposition \ref{prop1}. Assume that, for each fixed \(t\in[0,T]\), where \(T<1\) in the unregularized rectified-flow case, and \(T=1\) is allowed in variance-floored cases, the weight functions $z\mapsto w_i(t,z)$ are continuously differentiable. Then, the vector field $z\mapsto \hat{v}^*(t, z)$ is a gradient field on $\R^d$ if and only if
$$\sum_{i=1}^N \left( v_i(t, z) \nabla_z w_i(t, z)^\top - \nabla_z w_i(t, z) v_i(t, z)^\top \right) = 0 \quad \text{for all } z\in\R^d.$$

    \end{proposition}

In general, this identity is not expected to hold except in special symmetric or degenerate
configurations; explicit counterexamples can be constructed already in \(d=2\). Thus, wherever the Benamou--Brenier optimal velocity is characterized by a gradient field, the empirical minimizer cannot coincide with it unless the skew-symmetric condition vanishes. Intuitively, this says that even if every individual conditional flow is a straight line (gradient field), their weighted sum is not generally a gradient field because the weights $w_i(t,z)$ vary spatially (dependent on $z$).

An important consequence of Proposition \ref{prop2}, together with Proposition \ref{prop1}, is that the ideal empirical target velocity for neural CFM training is generally not a gradient field, even if the underlying population construction is formulated to be one.

\subsection{An Equivalent Class of Empirical Samplers}
\label{subsec:flux_equiv}

The preceding non-gradient result concerns the particular velocity field selected by the EFM square-loss objective. At the level of marginal density evolution, however, this representative is not unique. We now make this non-uniqueness explicit using probability fluxes. 

Let \(p_t\) be a smooth positive density on \(\R^d\). For a vector field
\(v_t\in L^2(p_t;\R^d)\), we define its \emph{probability flux}, or
probability current, by
\[
    j_t := p_t v_t .
\]
Here
\(
    L^2(p_t;\R^d)
    :=
    \left\{
    v:\R^d\to\R^d:
    \int_{\R^d}\|v(z)\|^2p_t(z)\,dz<\infty
    \right\}.
\)
With this notation, the continuity equation is
\[
    \partial_t p_t+\nabla\cdot j_t=0,
    \qquad\text{equivalently}\qquad
    \partial_t p_t+\nabla\cdot(p_t v_t)=0.
\]

We will use divergences in the weak, or distributional, sense. Since
\(v_t\in L^2(p_t;\R^d)\), the current \(p_t v_t\) belongs to
\(L^1_{\rm loc}(\R^d;\R^d)\). Hence \(\nabla\cdot(p_t v_t)\) is
well-defined as a distribution. In particular,
\[
    \nabla\cdot(p_t v_t)=0
    \quad\text{in }\mathcal D'(\R^d)
\]
means that
\[
    \int_{\R^d} p_t(z)v_t(z)\cdot\nabla\varphi(z)\,dz=0
\]
for every test function \(\varphi\in C_c^\infty(\R^d)\).

For fixed \(p_t\), define the \emph{flux-null remainder space}
\[
    \mathcal R_{p_t}
    :=
    \left\{
    r\in L^2(p_t;\R^d):
    \nabla\cdot(p_t r)=0
    \text{ in }\mathcal D'(\R^d)
    \right\}.
\]
Equivalently,
\[
    r\in\mathcal R_{p_t}
    \quad\Longleftrightarrow\quad
    \int_{\R^d}p_t(z)r(z)\cdot\nabla\varphi(z)\,dz=0
    \quad
    \forall \varphi\in C_c^\infty(\R^d).
\]
We call such remainders \emph{flux-null}, since they generate a probability
current \(p_t r\) with zero divergence. When \(p_t\) and \(r\) are smooth and
\(p_t>0\), this is equivalently the weighted divergence-free condition
\[
    \nabla\cdot r+r\cdot\nabla\log p_t=0.
\]
This condition is analogous to the gauge freedom studied for diffusion models~\cite{horvat2024gauge}, where non-conservative remainders can preserve the same marginal evolution under suitable flux conditions. Here, it describes non-uniqueness of particle
dynamics along a fixed EFM marginal path.

We say that two velocity fields \(u_t,v_t\in L^2(p_t;\R^d)\) are
\emph{flux-equivalent} with respect to \(p_t\), and write
\(u_t\sim_{p_t}v_t\), if
\[
    \nabla\cdot(p_tu_t)=\nabla\cdot(p_tv_t)
    \quad\text{in }\mathcal D'(\R^d).
\]
Equivalently,
\(
    u_t-v_t\in\mathcal R_{p_t}.
\)
The relation \(\sim_{p_t}\) is an equivalence relation, since it is
defined by equality of distributional divergences. Its equivalence class
at \(v_t\) is
\[
    [v_t]_{p_t}
    =
    \left\{
    u_t\in L^2(p_t;\R^d):
    u_t\sim_{p_t}v_t
    \right\}
    =
    v_t+\mathcal R_{p_t}.
\]
Thus flux equivalence identifies velocity fields that induce the same
marginal density evolution while allowing different particle trajectories. For a time-dependent path \(p=(p_t)_{t\in[0,T]}\), we write
\(u_\cdot\sim_p v_\cdot\) if \(u_t\sim_{p_t}v_t\) for a.e.
\(t\in[0,T]\).

The following result is a natural consequence of the above formulation. 

\begin{proposition}[Flux-equivalent empirical samplers]
\label{prop:flux_equiv}
Fix a finite time horizon \(T>0\). Let \((\hat p_t)_{t\in[0,T]}\) be a smooth positive empirical marginal path and suppose that \(\hat v_t\) satisfies
\[
    \partial_t\hat p_t+\nabla\cdot(\hat p_t\hat v_t)=0.
\]
If \(r_t\in\mathcal R_{\hat p_t}\) for a.e. \(t\), then
\[
    u_t=\hat v_t+r_t
\]
satisfies
\[
    \partial_t\hat p_t+\nabla\cdot(\hat p_tu_t)=0.
\]
Consequently, \(u_t\) and \(\hat v_t\) generate the same empirical marginal path at the level of the continuity equation. If the corresponding ODE flows are well posed and the continuity equation is unique in the chosen class, then both flows push \(\hat p_0\) forward to \(\hat p_t\).
\end{proposition}

The proposition should be read as a statement about the Eulerian marginal path. Flux-equivalent samplers may have different Lagrangian particle trajectories, different numerical stiffness, and different kinetic energies, even though their one-time marginals agree.

The notation \(r_t\) is chosen to emphasize that these fields are remainder directions: they change the velocity field while contributing a divergence-free probability current \(\hat p_t r_t\). Thus they change particle trajectories without changing the marginal density evolution. This condition is closely related to the gauge freedom condition for diffusion models studied in \cite{horvat2024gauge} (see also the related work cited there); here we only use the elementary flux interpretation and formalize this condition.

\paragraph{Projection onto gradient fields.}
The next observation gives a canonical
representative from a flux-equivalence class. The flux-null remainder space is the orthogonal complement of gradient fields in \(L^2(p_t;\R^d)\). Let
\[
    \mathcal G_{p_t}
    :=
    \overline{\{\nabla\phi:\phi\in C_c^\infty(\R^d)\}}^{L^2(p_t)}.
\]

Integration by parts gives
\[
    \langle r,\nabla\phi\rangle_{p_t}
    =
    \int r\cdot\nabla\phi\,p_t\,dz
    =
    -\int \phi\,\nabla\cdot(p_t r)\,dz.
\]
Since \(\mathcal G_{p_t}\) is closed by definition, the Hilbert projection theorem gives an orthogonal decomposition of \(L^2(p_t;\mathbb R^d)\) into \(\mathcal G_{p_t}\) and \(\mathcal G_{p_t}^{\perp}\). 
Moreover, by the weak
definition of divergence,
\[
    r\in\mathcal R_{p_t}
    \iff
    \int_{\mathbb R^d} p_t(z)r(z)\cdot\nabla\varphi(z)\,dz=0
    \quad
    \forall \varphi\in C_c^\infty(\mathbb R^d).
\]
Hence \(\mathcal R_{p_t}=\mathcal G_{p_t}^{\perp}\), and every
\(v_t\in L^2(p_t;\mathbb R^d)\) admits the orthogonal decomposition
\[
    v_t=P_{\mathcal G_{p_t}}v_t+P_{\mathcal R_{p_t}}v_t.
\]

When the projection onto gradient fields has a smooth potential, we write \(P_{\mathcal G_{p_t}}v_t=\nabla\phi_t\), and so
\[
    v_t=\nabla\phi_t+r_t,
    \qquad r_t\in\mathcal R_{p_t},
\]
where \(\phi_t\) solves
\[
    \nabla\cdot(p_t\nabla\phi_t)=\nabla\cdot(p_t v_t)
\]
in the weak sense. The field \(\nabla\phi_t\) is the minimum kinetic energy representative of the fixed equivalence class \([v_t]_{p_t}\). Indeed, any other representative has the form \(\nabla\phi_t+r\) with \(r\in\mathcal R_{p_t}\), and orthogonality gives
\[
    \|\nabla\phi_t+r\|_{L^2(p_t)}^2
    =\|\nabla\phi_t\|_{L^2(p_t)}^2+\|r\|_{L^2(p_t)}^2.
\]

This fixed-path statement should not be confused with the full Benamou--Brenier problem. The latter optimizes over both \((p_t)\) and \((v_t)\). Here the empirical path \((\hat p_t)\) is fixed, and we only consider optimizing over velocity representatives that realize the same path.

\paragraph{Explicit flux-null corrections for Gaussian empirical paths.}
For Gaussian empirical affine paths,  one can construct a useful
subfamily of flux-null directions explicitly. Here we allow Gaussian affine paths with matrix-valued covariances.

\begin{proposition}[Explicit Gaussian flux-null corrections]
\label{prop:gaussian_flux_null}
Fix \(t\) and suppose that the empirical marginal density is a finite Gaussian
mixture
\[
    \hat p_t(z)
    =
    \frac1N\sum_{i=1}^N p_i(t,z),
    \qquad
    p_i(t,z)=\mathcal N(z;m_i(t),\Sigma_i(t)),
\]
where each \(\Sigma_i(t)\) is symmetric positive definite. Define
\[
    w_i(t,z)
    =
    \frac{p_i(t,z)}{\sum_{j=1}^N p_j(t,z)}.
\]
For any collection of antisymmetric matrices \(A_i(t)^\top=-A_i(t)\), define
\[
    r_t^A(z)
    :=
    \sum_{i=1}^N
    w_i(t,z)\,\Sigma_i(t)A_i(t)(z-m_i(t)).
\]
If \(r_t^A\in L^2(\hat p_t;\mathbb R^d)\), then \(r_t^A\) is flux-null with
respect to \(\hat p_t\); i.e.,
\(
    r_t^A\in\mathcal R_{\hat p_t},
    \ 
    \nabla\cdot(\hat p_t r_t^A)=0
\)
in the distributional sense.
\end{proposition}

This proposition shows that antisymmetric rotations inside each Gaussian
component generate probability currents whose total divergence vanishes. Hence
adding \(r_t^A\) changes particle trajectories but not the empirical density
tangent.

For variance-floored rectified flow, we have
\[
    m_i(t)=t x^{(i)},
    \qquad
    \Sigma_i(t)=\sigma_t^2 I_d,
    \qquad
    \sigma_t=1-(1-\sigma_{\min})t.
\]
Thus Proposition~\ref{prop:gaussian_flux_null} yields the explicit flux-null
family
\[
    r_t^A(z)
    =
    \sigma_t^2
    \sum_{i=1}^N
    w_i(t,z)A_i(t)(z-tx^{(i)}),
    \qquad
    A_i(t)^\top=-A_i(t).
\]
Consequently, every velocity field \(u_t^A=\hat v_t+r_t^A\) realizes the same
variance-floored empirical marginal path as \(\hat v_t\) at the level of the
continuity equation. For the variance-floored rectified-flow minimizer,
\[
    \hat v_t(z)
    =
    \sum_{i=1}^N
    w_i(t,z)
    \frac{x^{(i)}-(1-\sigma_{\min})z}{\sigma_t},
\]
this gives
\[
    u_t^A(z)
    =
    \sum_{i=1}^N
    w_i(t,z)
    \frac{x^{(i)}-(1-\sigma_{\min})z}{\sigma_t}
    +
    \sigma_t^2
    \sum_{i=1}^N
    w_i(t,z)A_i(t)(z-tx^{(i)}).
\]

For unregularized rectified flow, \(\sigma_t=1-t\) degenerates at \(t=1\), so
the smooth-density statements above should be read on compact intervals
\([0,T]\subset[0,1)\). With \(\sigma_{\min}>0\), the mixture remains smooth and
positive on the full interval \([0,1]\).

The antisymmetric-matrix construction is not a complete parameterization of
\(\mathcal R_{\hat p_t}\). It gives an explicit finite-dimensional flux-null
subfamily. More generally, the full space can be described through
divergence-free currents \(j_t\) satisfying \(\nabla\cdot j_t=0\), together
with sufficient integrability so that \(r_t=j_t/\hat p_t\in
L^2(\hat p_t;\mathbb R^d)\). In two dimensions, such currents may be represented
by stream functions \(j_t=J\nabla\psi_t\) under suitable 
assumptions; in higher dimensions, one may use antisymmetric tensor potentials.

\subsection{Kinetic Energy Tail-Bounds}
Quantifying the kinetic behavior of population and empirical FM samplers is a natural way to understand how often high-energy trajectories arise and what mechanisms produce them.

First, we focus on the Gaussian rectified flow (RF) example in Example \ref{ex_gaussianRF}, which is tractable enough to allow for   precise analysis. The following result shows that the probability of a generated sample under the population RF model that has high kinetic energy decays exponentially. Since this is the OT map and velocity is constant along straight paths, this bound applies simultaneously to the instantaneous kinetic energy at any time $t$ and the integrated total energy.

\begin{proposition}[Population setting, OT case] \label{prop3}
    Let $p_0 = \mathcal{N}(0, I_d)$ and $p_1 = \mathcal{N}(m_1, \Sigma_1)$, where $\Sigma_1$ is positive definite. Let $R(x) = m_1 + \Sigma_1^{1/2}x$ be the rectified flow map from Example \ref{ex_gaussianRF}. For a generated sample $Y \sim p_1$, let $E(Y)
=
\int_0^1
\left\|
v\left(t,\psi_t(R^{-1}(Y))\right)
\right\|^2dt = \| Y - R^{-1}(Y)\|^2$ be the random variable representing the kinetic energy (integrated or instantaneous).

    (a) For all $y \in \R^d$, $\frac{1}{2} E(y) = - \log p_1(y) + C(y)$, where
        \begin{equation}
        C(y) =  \frac{1}{2} y^T (I_d - 2 \Sigma_1^{-1/2}) y + m_1^T \Sigma_1^{-1/2} y - \frac{1}{2} \log \det(2\pi \Sigma_1).  
    \end{equation}
    
    (b) Assume $\Sigma_1\neq I_d$. Let $\lambda_i(\Sigma_1)$ denote the eigenvalues of $\Sigma_1$, and define
\[
\rho
:=
\max_{i=1,\dots,d}
\bigl(\sqrt{\lambda_i(\Sigma_1)}-1\bigr)^2
>0.
\]
Then, for every $u>0$,
\[
\mathbb P_{Y\sim p_1}\!\bigl(E(Y)\ge u\bigr)
\le
C\,\exp\!\left(-\frac{u}{4\rho}\right),
\]
where
\[
C
=
2^{d/2}\exp\!\left(\frac{\|m_1\|^2}{2\rho}\right).
\]
If \(\Sigma_1=I_d\), then \(E(Y)=\|m_1\|^2\) is deterministic, so the tail bound is trivial.
\end{proposition}

Part (a) shows that in this Gaussian OT/RF case, kinetic energy differs from the target negative log-density by an explicit quadratic correction. Part (b) shows that high-energy samples are exponentially unlikely under \(p_1\). Importantly, this phenomenon arises purely from the design of the Gaussian RF model itself and the assumption that $p_1$ is Gaussian.



A similar exponential upper-tail bound holds for the empirical RF model
conditional on any fixed finite dataset, even though the empirical velocity is
nonlinear and generally not OT-optimal.

\begin{theorem}[Empirical setting, Gaussian source] \label{thm1}
Let $X_0 \sim \mathcal{N}(0, I_d)$ and suppose that we are given a fixed dataset $\mathcal{D}_N = \{x^{(i)}\}_{i \in [N]}$, $x^{(i)} \in \R^d$, with $M := \max_i \|x^{(i)}\| < \infty$.
Let $T \in [0,1)$ and define the instantaneous kinetic energy $K_t = \|\hat{v}^*(t, \psi_t(X_0))\|^2$ and the corresponding time-integrated kinetic energy $E_T = \int_0^T K_t \, dt$, where $\hat{v}^*$ is given in \eqref{eq_empRF} and $\psi_t$ solves $\dot{\psi}_t(X) = \hat{v}^*(t, \psi_t(X))$, $\psi_0(X) = X_0$, for $t \in [0, 1)$. Assume that there exists a unique solution to this ODE on $[0,T]$.

    \begin{itemize}
        \item[(a)] For each $t \in [0, T]$, there exist constants $C_t>0$, $c_t>0$, and threshold $U_t\ge 0$, depending only on $t$, $d$ and $M$, such that for every $u \ge U_t$,
        $$\mathbb{P}(K_t \ge u \mid \mathcal{D}_N) \le C_t \, e^{-c_t u}. $$
        \item[(b)] There exist constants $C_T>0$, $c_T>0$, and threshold  $U_T\ge 0$, depending only on $T$, $d$ and $M$, such that for every $u \ge U_T$,
        $$\mathbb{P}(E_T \ge u \mid \mathcal{D}_N) \le C_T \, e^{-c_T u}. $$
    \end{itemize} 
\end{theorem}


Theorem \ref{thm1} implies that, just as in the population case, both instantaneous and integrated empirical kinetic energies satisfy exponential upper-tail bounds beyond a sufficiently large threshold. This phenomenon is driven by the Gaussian source distribution and holds regardless of whether the velocity field is OT-optimal.

The above bounds are conditional on the realized finite dataset; the only randomness is the draw \(X_0\sim\mathcal N(0,I_d)\). Hence even if the data points were sampled from a heavy-tailed target distribution, the exact empirical RF sampler with Gaussian source still satisfies exponential  energy upper-tail bounds on every interval \([0,T]\), \(T<1\). To obtain polynomial energy tails, one must modify the source distribution itself rather than merely perturb the observed data points.

Indeed, while Theorem \ref{thm1} establishes exponential upper-tail bounds  due to the Gaussian source, the empirical framework allows for heavy-tailed modeling if we instead consider a smoothed model from Example \ref{eg_empaffineflows} and choose the source kernel $K$ to be heavy-tailed. Specifically, if \(X_0\sim K\) satisfies a polynomial upper-tail bound \(P(\|X_0\|>s)\le C_\alpha s^{-\alpha}\), then the linear growth of the vector field propagates this polynomial control to the kinetic energy. This gives the polynomial upper-tail bound in the following theorem. 

\begin{theorem}[Empirical setting, polynomial source-tail upper bound]
\label{thm2_polydecay}
Let $D_N = \{x^{(i)}\}_{i\in[N]}$ be a fixed  dataset with 
$\max_{i\in[N]} \|x^{(i)}\| < \infty$. Let $T \in [0,1)$.
Suppose the source distribution \(p_0\) satisfies the polynomial upper-tail bound:
\[
\mathbb{P}(\|X_0\| \ge s) \le \frac{C_\alpha}{s^\alpha}
\quad\text{for all } s \ge 1,
\]
for some constants $C_\alpha > 0$ and tail index $\alpha > 0$.

For the velocity field $\hat v^\ast$ defined in Proposition~\ref{prop1}, let
\[
A_{\max} := \sup_{t\in[0,T],\,i\in[N]} |a_t(x^{(i)})|,
\qquad
B_{\max} := \sup_{t\in[0,T],\,i\in[N]} \|b_t(x^{(i)})\|,
\]
and assume that \(A_{\max}<\infty\), \(B_{\max}<\infty\), and that there exists a unique solution to the ODE driven by $\hat v^\ast$
on $[0,T]$. Then, for each $t\in[0,T]$, there exist constants $C_t>0$ and threshold $U_t\ge 0$, depending only on $t, T, A_{\max}, B_{\max}, C_\alpha, \alpha$, such that for every $u\ge U_t$,
\[
\mathbb{P}(K_t \ge u \mid D_N) \le \frac{C_t}{u^{\alpha/2}}.
\]
Moreover, there exist constants $C_T>0$ and threshold $V_T\ge 0$, depending only on $T, A_{\max}, B_{\max}$, $C_\alpha$, $\alpha$, such that for every $u\ge V_T$,
\[
\mathbb{P}(E_T \ge u \mid D_N) \le \frac{C_T}{u^{\alpha/2}}.
\]
\end{theorem}

This shows that polynomial source-tail upper bounds propagate to polynomial energy-tail upper bounds. Establishing matching lower bounds would require additional nondegeneracy assumptions on the affine coefficients.
The source distribution therefore provides a primary mechanism controlling the upper tails of kinetic energy.

\subsection{Tail Bounds for Flux-Equivalent Representatives}
\label{subsec:flux_tail}

The preceding tail bounds are not specific to the square-loss representative \(\hat v^\ast\). They depend on a linear-growth estimate for the velocity field. Thus they extend to any flux-equivalent representative whose flux-null remainder has controlled growth.

The following result is a linear-growth consequence and does not depend on the detailed form of the kernel beyond the affine velocity bound.

\begin{proposition}[Linear growth implies source-tail upper bounds]
\label{lem:linear_growth_tail}
Fix a finite time horizon \(T>0\). Let \(u_t\) be a time-dependent velocity field on
\([0,T]\) whose ODE flow is well posed. Suppose there exist constants
\(L_T,B_T\ge0\) such that
\[
    \|u_t(z)\|\le L_T\|z\|+B_T,
    \qquad t\in[0,T],\ z\in\R^d.
\]
Let \(X_t\) solve the ODE
\(
    \dot X_t=u_t(X_t),
    \  X_0\sim p_0,
\)
and define
\[
    K_t^u:=\|u_t(X_t)\|^2,
    \qquad
    E_T^u:=\int_0^T K_t^u\,dt.
\]
Then there exists a constant \(C_T>0\), depending only on
\(T,L_T,B_T\), such that, for all \(t\in[0,T]\),
\[
    K_t^u\le C_T(1+\|X_0\|^2),
    \qquad
    E_T^u\le C_T(1+\|X_0\|^2).
\]
Consequently, if \(X_0\sim\mathcal N(0,I_d)\), then there exist constants
\(c,C>0\), depending on \(T,L_T,B_T\) and \(d\), such that for all
sufficiently large \(\lambda\),
\[
    \mathbb P(K_t^u\ge \lambda)\le C e^{-c\lambda},
    \qquad
    \mathbb P(E_T^u\ge \lambda)\le C e^{-c\lambda}.
\]
If instead
\[
    \mathbb P(\|X_0\|\ge s)\le C_\alpha s^{-\alpha},
    \qquad\text{for all }s\ge1,
\]
then there exists a constant \(C>0\), depending on
\(T,L_T,B_T,C_\alpha,\alpha\), such that for all sufficiently large
\(\lambda\),
\[
    \mathbb P(K_t^u\ge \lambda)\le C\lambda^{-\alpha/2},
    \qquad
    \mathbb P(E_T^u\ge \lambda)\le C\lambda^{-\alpha/2}.
\]
\end{proposition}

We now verify that the explicit Gaussian flux-null representatives satisfy the linear-growth condition of Proposition~\ref{lem:linear_growth_tail}.

\begin{theorem}[Flux-equivalent empirical affine samplers]
\label{thm:flux_equiv_affine_tail}
Fix a finite time horizon \(T>0\), and let
\[
    \hat p_t(z)
    =
    \frac1N\sum_{i=1}^N
    \mathcal N(z;m_i(t),\Sigma_i(t)),
    \qquad t\in[0,T],
\]
where each \(\Sigma_i(t)\) is symmetric positive definite. Define
\[
    w_i(t,z)
    =
    \frac{\mathcal N(z;m_i(t),\Sigma_i(t))}
    {\sum_{j=1}^N\mathcal N(z;m_j(t),\Sigma_j(t))}.
\]
Suppose the empirical affine FM velocity is
\[
    \hat v_t(z)
    =
    \sum_{i=1}^N
    w_i(t,z)(B_i(t)z+b_i(t)),
\]
and satisfies \(\partial_t\hat p_t+\nabla\cdot(\hat p_t\hat v_t)=0\).
Let \(A_i(t)^\top=-A_i(t)\), and define
\[
    r_t^A(z)
    =
    \sum_{i=1}^N
    w_i(t,z)\Sigma_i(t)A_i(t)(z-m_i(t)),
    \qquad
    u_t^A(z)=\hat v_t(z)+r_t^A(z).
\]

Assume the ODE driven by \(u_t^A\) is well posed and that
\[
    M_T:=\sup_{i,t\in[0,T]}\|m_i(t)\|<\infty,
    \qquad
    B_T^{\rm aff}:=\sup_{i,t\in[0,T]}\|B_i(t)\|_{\mathrm{op}}<\infty,
\]
\[
    b_T^{\rm aff}:=\sup_{i,t\in[0,T]}\|b_i(t)\|<\infty,
    \qquad
    R_T^A:=\sup_{i,t\in[0,T]}\|\Sigma_i(t)A_i(t)\|_{\mathrm{op}}<\infty.
\]
Then \(\nabla\cdot(\hat p_t r_t^A)=0\) and
\(r_t^A\in L^2(\hat p_t;\mathbb R^d)\), hence \(u_t^A\) is flux-equivalent
to \(\hat v_t\) and satisfies
\(
    \partial_t\hat p_t+\nabla\cdot(\hat p_tu_t^A)=0.
\)
Moreover, \(u_t^A\) satisfies
\[
    \|u_t^A(z)\|
    \le
    L_T^A\|z\|+B_T^A,
    \qquad
    L_T^A:=B_T^{\rm aff}+R_T^A,
    \quad
    B_T^A:=b_T^{\rm aff}+R_T^A M_T.
\]

\end{theorem}

Consequently, Proposition~\ref{lem:linear_growth_tail} applies to the \(u_t^A\) defined above.
In particular, the deterministic energy bounds and the Gaussian or polynomial
source-tail upper bounds in Proposition~\ref{lem:linear_growth_tail} hold with
\(L_T=L_T^A\) and \(B_T=B_T^A\).

Finally, we specialize this theorem to the example of empirical rectified flow.  
\begin{corollary}[Variance-floored empirical rectified flow]
\label{cor:flux_equiv_vfrf}
Let \(T\in(0,1]\), \(\sigma_t=1-(1-\sigma_{\min})t\), and
\(\sigma_{\min}>0\). For \(t\in[0,T]\), let
\[
    \hat p_t(z)
    =
    \frac1N\sum_{i=1}^N
    \mathcal N(z;tx^{(i)},\sigma_t^2I).
\]
The empirical variance-floored rectified-flow velocity is
\[
    \hat v_t(z)
    =
    \sum_{i=1}^N
    w_i(t,z)
    \frac{x^{(i)}-(1-\sigma_{\min})z}{\sigma_t}.
\]
Let \(A_i(t)^\top=-A_i(t)\), and define
\[
    r_t^A(z)
    =
    \sigma_t^2
    \sum_{i=1}^N
    w_i(t,z)A_i(t)(z-tx^{(i)}),
    \qquad
    u_t^A(z)=\hat v_t(z)+r_t^A(z).
\]
If \(A_{\max}:=\sup_{i,t\in[0,T]}\|A_i(t)\|_{\mathrm{op}}<\infty\), then
\(u_t^A\) is flux-equivalent to \(\hat v_t\), generates the same empirical
marginal path, and satisfies
\[
    \|u_t^A(z)\|
    \le
    L_T^A\|z\|+B_T^A,
\]
where one may take
\[
    L_T^A
    =
    \frac{1-\sigma_{\min}}{\sigma_{\min}}
    +
    A_{\max},
    \qquad
    B_T^A
    =
    \frac{M}{\sigma_{\min}}
    +
    A_{\max}M,
    \qquad
    M:=\max_{i\in[N]}\|x^{(i)}\|.
\]
Consequently, the deterministic and source-tail bounds of
Proposition~\ref{lem:linear_growth_tail} hold for \(u_t^A\).
\end{corollary}

We end with an important caveat. Without growth control, flux-null modifications can arbitrarily alter kinetic
energy tails while preserving the same marginal path. The goal of the above
results is therefore not to show that all flux-equivalent representatives have
the same tails, but rather that the Gaussian and polynomial source-tail upper
bounds persist for representatives with controlled linear growth.  Flux equivalence alone does not control kinetic energy, as shown by the following remark, which could potentially give us a new angle to understand memorization vs. generalization in FM \cite{li2026kinetic}.

\begin{remark}
Flux equivalence preserves the marginal density evolution at the level of the
continuity equation, but it does not by itself control particle speeds or
kinetic energy.

To see this, consider the two-dimensional variance-floored empirical rectified
flow with one data point \(x^{(1)}=0\). Then
\[
    \hat p_t=\mathcal N(0,\sigma_t^2 I_2),
    \qquad
    \sigma_t=1-(1-\sigma_{\min})t,
\]
and the standard empirical velocity is
\(
    \hat v_t(z)
    =
    -\frac{1-\sigma_{\min}}{\sigma_t}z.
\)
Let
\(
    J=
    \begin{pmatrix}
        0 & -1 \\
        1 & 0
    \end{pmatrix}
\)
be the \(90^\circ\) rotation matrix. For \(a\in(0,1/4)\), define
\[
    r_t(z)
    =
    \exp\!\left(a\frac{\|z\|^2}{\sigma_t^2}\right)Jz.
\]
For each fixed \(t\), we have \(r_t\in L^2(\hat p_t;\mathbb R^2)\) and
\(
    \nabla\cdot(\hat p_t r_t)=0.
\)
Indeed, writing \(z=(x,y)\) and \(s=\|z\|^2/\sigma_t^2\), the current
\(\hat p_t r_t\) has the form
\[
    \hat p_t(z)r_t(z)
    =
    q_t(s)(-y,x)
\]
for a scalar radial function \(q_t\). Hence
\[
\begin{aligned}
    \nabla\cdot(\hat p_t r_t)
    &=
    \partial_x[-yq_t(s)]
    +
    \partial_y[xq_t(s)] =
    -yq_t'(s)\frac{2x}{\sigma_t^2}
    +
    xq_t'(s)\frac{2y}{\sigma_t^2}
    =
    0.
\end{aligned}
\]
Moreover, if \(Z_t\sim\hat p_t\) and
\(
    S:=\frac{\|Z_t\|^2}{\sigma_t^2},
\)
then \(S\sim\chi_2^2\), equivalently \(S\) is exponential with rate \(1/2\).
Since \(\|Jz\|=\|z\|\),
\[
    \mathbb E_{\hat p_t}\|r_t(Z_t)\|^2
    =
    \sigma_t^2\mathbb E\left[S e^{2aS}\right]
    =
    \frac{\sigma_t^2}{2}
    \int_0^\infty s e^{-(1/2-2a)s}\,ds
    <
    \infty
\]
because \(a<1/4\). Thus \(r_t\in L^2(\hat p_t;\mathbb R^2)\).

Therefore \(u_t=\hat v_t+r_t\) is flux-equivalent to \(\hat v_t\), and so it
preserves the same empirical marginal density evolution at the level of the
continuity equation. However, the instantaneous kinetic energy can have a much
heavier tail. Since \(\hat v_t(z)\) is radial and \(r_t(z)\) is rotational,
\(
    \hat v_t(z)\cdot r_t(z)=0.
\)
Consequently,
\[
\begin{aligned}
    \|u_t(Z_t)\|^2
    &=
    \|\hat v_t(Z_t)\|^2+\|r_t(Z_t)\|^2 =
    (1-\sigma_{\min})^2S
    +
    \sigma_t^2 S e^{2aS}.
\end{aligned}
\]
The first term has an exponential tail, whereas the second term has
polynomial-type tail decay up to logarithmic factors. Indeed, if \(s_\lambda\)
is defined by
\(
    \sigma_t^2 s_\lambda e^{2a s_\lambda}=\lambda,
\)
then
\[
    \mathbb P\!\left(\sigma_t^2 S e^{2aS}\ge \lambda\right)
    =
    \mathbb P(S\ge s_\lambda)
    =
    e^{-s_\lambda/2}.
\]
The solution is
\(
    s_\lambda
    =
    \frac{1}{2a}
    W\!\left(\frac{2a\lambda}{\sigma_t^2}\right),
\)
where \(W\) is the Lambert \(W\)-function. Since \(W(x)\sim\log x\) as
\(x\to\infty\), this tail behaves like a polynomial in \(\lambda\), up to
logarithmic corrections. Thus the same empirical marginal path can be realized
by flux-equivalent velocities with very different kinetic energy tails.
\end{remark}

\section{Numerical Validation}
\label{sec:numerics}

We complement the theoretical results with toy  experiments illustrating the source-driven
kinetic energy behavior predicted by Theorems~\ref{thm1} and~\ref{thm2_polydecay}. The goal is
not to benchmark generative quality, but to test the qualitative mechanism suggested by the
theory: conditional on a fixed dataset, the upper-tail behavior of the kinetic energy is
controlled by the source distribution.

We consider two experiments. First, we simulate the exact empirical affine-flow minimizer from
Proposition~\ref{prop1} on three two-dimensional toy datasets: two moons, eight Gaussian
clusters, and a checkerboard distribution. We compare a Gaussian source with coordinate-wise
Student-\(t_\nu\) sources for \(\nu\in\{2,5,10\}\). For each dataset and source, we generate
trajectories by solving
$\dot Z_t = \hat v(t,Z_t),$
where, for the regularized affine path
\(
    s_t=1-(1-\sigma_{\min})t, \ 
    m_t(x^{(i)})=t x^{(i)},
\)
the exact empirical minimizer is
\[
    \hat v(t,z)
    =
    \sum_{i=1}^N w_i(t,z)
    \frac{x^{(i)}-(1-\sigma_{\min})z}{s_t},
\]
with posterior weights
\[
    w_i(t,z)
    =
    \frac{
    K\!\left((z-tx^{(i)})/s_t\right)
    }{
    \sum_{j=1}^N
    K\!\left((z-tx^{(j)})/s_t\right)
    }.
\]
We record the integrated kinetic energy
$E_T
    =
    \int_0^T
    \|\hat v(t,Z_t)\|^2\,dt.$
The numerical trajectories are computed using forward Euler, so the reported energies are
discretized approximations to the continuous-time quantities in the theory. Full implementation
details are given in Appendix~\ref{app:numerics}. 

Figure~\ref{fig:empirical_ET_survival} shows empirical survival curves for \(E_T\). Across all
three datasets, the Gaussian source produces the lightest upper tails, while Student-\(t\) sources
produce heavier tails, with heavier tails as \(\nu\) decreases. The target dataset affects the scale
of the energies, but the ordering of tail heaviness is stable across datasets and is primarily
determined by the source distribution.

\begin{figure}[!h]
\centering
\includegraphics[width=0.95\textwidth]{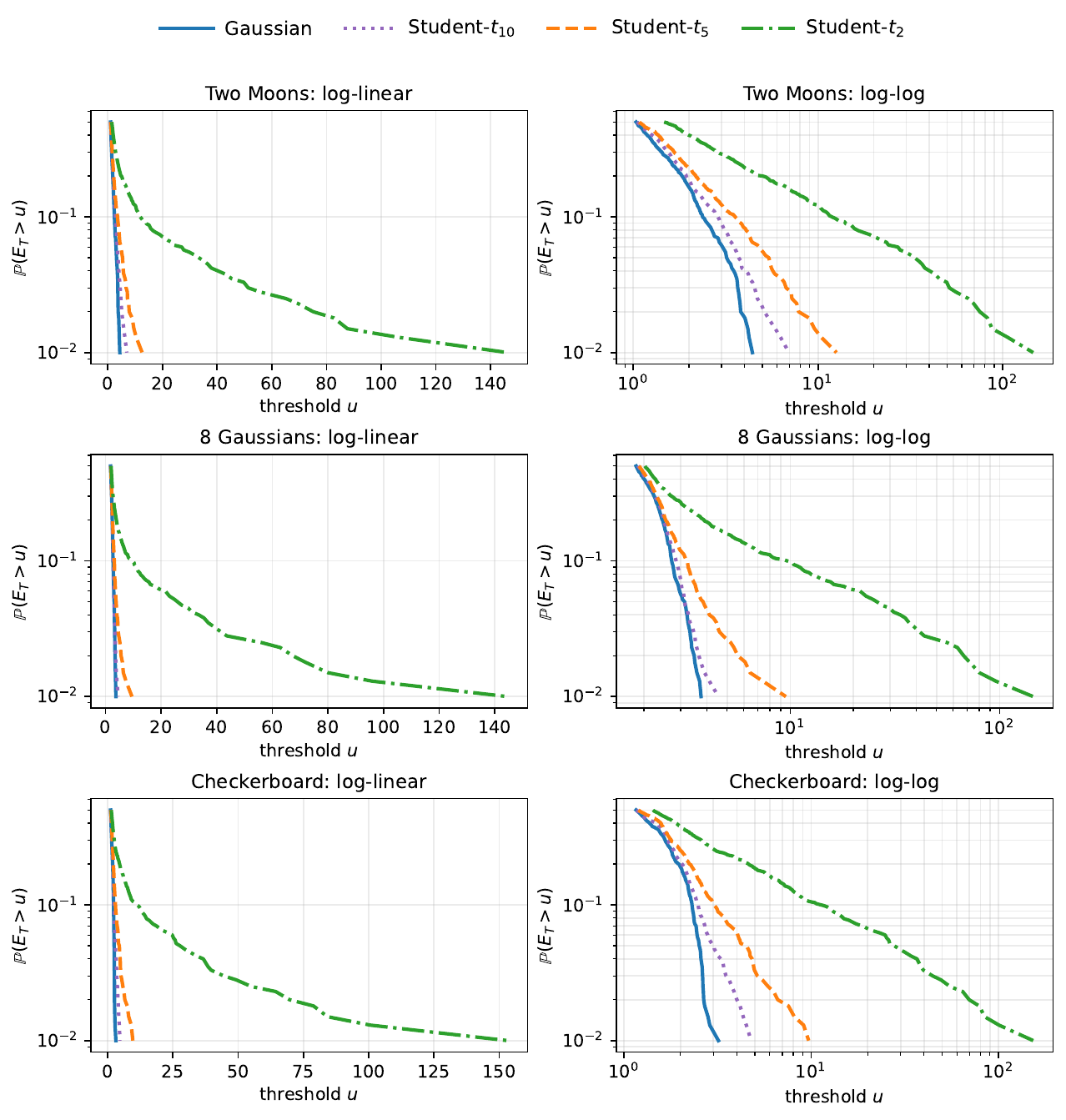}
\caption{
Empirical survival curves for the integrated kinetic energy \(E_T\) of the exact empirical
affine-flow sampler. Gaussian bases produce light upper tails, while Student-\(t\) bases produce
heavier tails as \(\nu\) decreases. The ordering is stable across datasets and is primarily
controlled by the source distribution.
}
\label{fig:empirical_ET_survival}
\end{figure}

Figure~\ref{fig:empirical_E_q99} summarizes the same effect through the empirical \(99\%\)
quantile of \(E_T\), averaged over random seeds. Heavy-tailed sources produce substantially
larger high-energy quantiles, consistent with the polynomial upper-tail mechanism in
Theorem~\ref{thm2_polydecay}.

\begin{figure}[!h]
\centering
\includegraphics[width=0.95\textwidth]{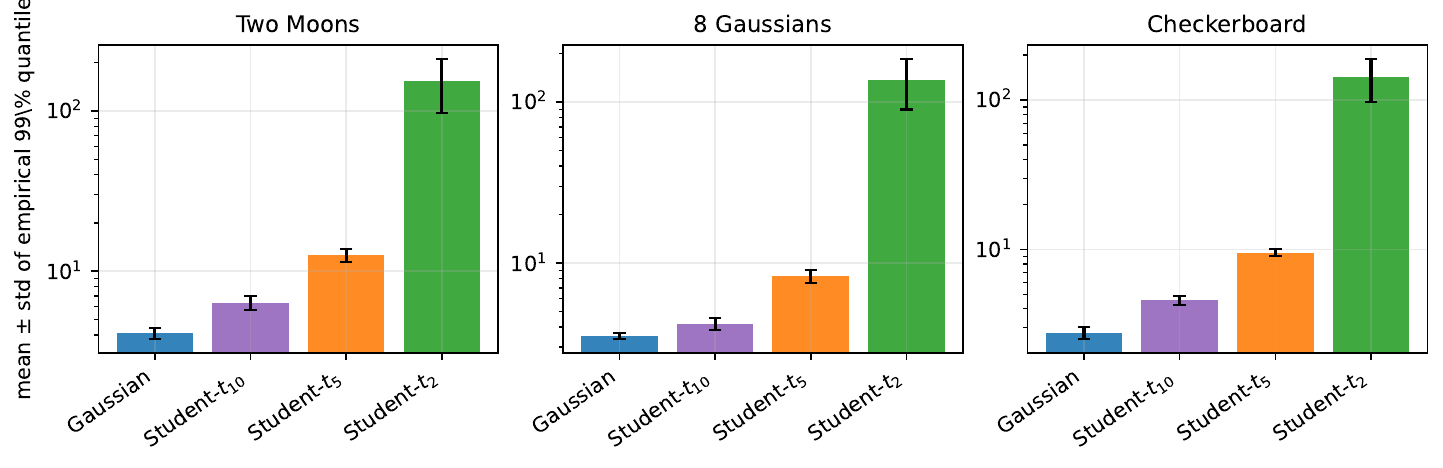}
\caption{
Mean and standard deviation across random seeds of the empirical \(99\%\) quantile of the
integrated kinetic energy \(E_T\). Heavy-tailed bases produce substantially larger high-energy
quantiles.
}
\label{fig:empirical_E_q99}
\end{figure}

Second, we isolate the sharpness of the polynomial source-to-energy exponent using a
nondegenerate affine ODE
$\dot Z_t=AZ_t+b.$
In this case,
\(
    E_T
    =
    \int_0^T
    \|AZ_t+b\|^2\,dt
    =
    (AX_0+b)^\top G_T(AX_0+b),
\)
where
$G_T
    =
    \int_0^T e^{tA^\top}e^{tA}\,dt.$
When \(A\) is nonsingular and \(G_T\) is positive definite, \(E_T\asymp \|X_0\|^2\) in the tail.
Thus a source tail of order \(s^{-\alpha}\) naturally induces an energy tail of order
\(u^{-\alpha/2}\). 
Figure~\ref{fig:sharpness_survival} shows that the heaviest-tailed case closely follows the benchmark exponent \(-\nu/2\), while lighter-tailed cases exhibit pre-asymptotic behavior over the plotted range. 

\begin{figure}[!h]
\centering
\includegraphics[width=0.95\textwidth]{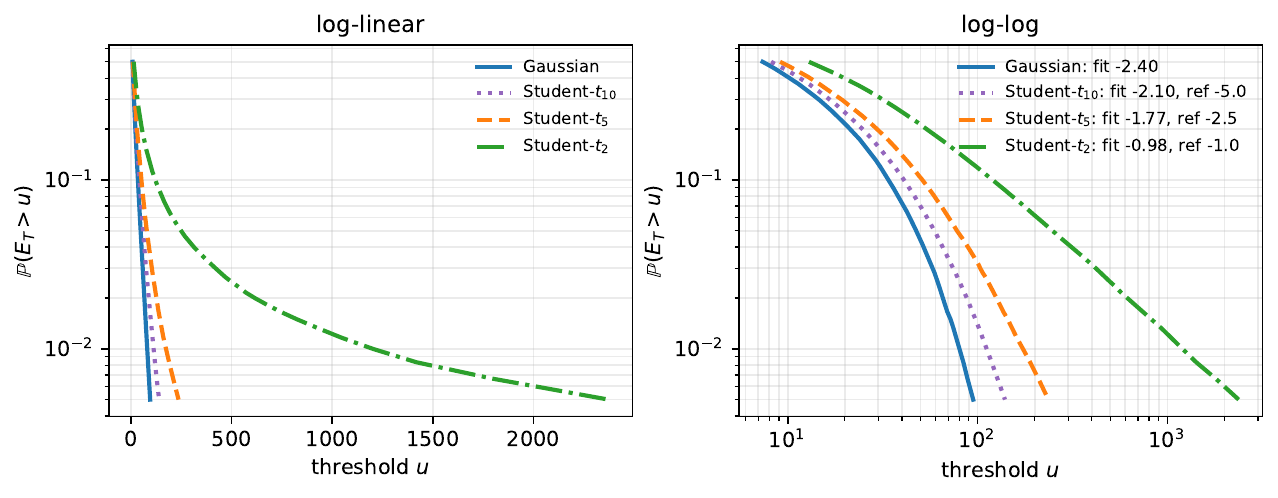}
\caption{
Survival curves for the nondegenerate affine sharpness experiment. For Student-\(t_\nu\) sources, the heaviest-tailed case \(t_2\) closely matches the benchmark \(-\nu/2\), while lighter-tailed cases require more extreme-tail resolution. The experiment supports the source-to-energy exponent mechanism in this controlled affine setting.
}
\label{fig:sharpness_survival}
\end{figure}

Overall, the experiments support the theoretical picture: EFM samplers inherit
energetic biases from the source distribution used to initialize them. Gaussian sources
produce light energy tails, while polynomially tailed sources produce heavier high-energy
profiles.

\section{Conclusion} \label{sec:discussion}

We proposed a plug-in perspective on flow matching that distinguishes objective-level empirical approximation from replacing the target law itself by raw empirical or smoothed finite-sample surrogates. This hierarchy shows that finite-sample FM is not merely population FM trained with Monte Carlo noise: it can change the statistical target, the transport geometry, and the energetic behavior of the sampler.

For affine conditional flows, we derived the exact empirical minimizer as a posterior-weighted mixture of conditional velocities. In the regularized affine setting, the terminal law is exactly a kernel density estimator, directly connecting smoothed empirical target FM with classical nonparametric density estimation and identifying the terminal scale as a bandwidth parameter.

We also identified a geometric bias of raw empirical target FM. Even when each conditional velocity is a gradient field, the empirical minimizer is generally not, because the posterior weights vary spatially. This gives a precise obstruction to Benamou--Brenier optimality and shows how empirical FM can introduce rotational components absent from optimal transport flows.

A further consequence is that the empirical marginal path does not determine a unique particle dynamics. We made this explicit through a probability-flux equivalence relation: two velocities are equivalent if their probability fluxes have the same divergence against the empirical marginal. The square-loss empirical FM minimizer is one representative of this class. Adding a flux-null remainder field \(r_t\) satisfying \(\nabla\cdot(\hat p_t r_t)=0\) preserves the empirical density path while changing particle trajectories. For variance-floored rectified flow and Gaussian affine conditional paths, we gave explicit flux-null subfamilies parameterized by antisymmetric matrices, together with a variational least-energy principle for selecting representatives.

Finally, we studied kinetic energy tails. Conditional on a fixed finite dataset, Gaussian sources yield exponential upper-tail bounds for instantaneous and integrated energies, while polynomially tailed sources yield corresponding polynomial bounds. The same qualitative source-controlled upper-tail mechanism extends to flux-equivalent representatives under bounded linear-growth assumptions on the flux-null remainder. Toy numerical experiments support this picture.

Overall, EFM exhibits several coupled finite-sample effects: a statistical plug-in bias from the surrogate target law, a geometric bias from posterior-weighted velocity mixtures, a non-uniqueness of particle dynamics modulo flux-null remainders, and an energetic bias controlled by the source distribution. Understanding how these effects persist under model (neural network) approximation, discretization, stochastic sampling, more general conditional paths, and for other data generating settings \cite{lim2024elucidating, lim2026flow} is an important direction for future work, as is designing source distributions, numerical schemes, timestep schedules \cite{gupta2026sharpenflowsharpnessawaresampling}, or flux-null remainder corrections that control energy profiles and trajectory-level behavior. Finally, it would also be interesting to study how the statistical errors of the plug-in estimators behave for different regimes and settings, which we leave to a future work.

{\bf Limitations.} Our analysis concerns exact empirical minimizers over unrestricted function classes. In practice, FM models are trained with neural networks and numerical ODE solvers are used during sampling. These approximations may introduce additional biases beyond the plug-in effects studied here. Moreover, our kinetic energy bounds are upper-tail results; matching lower bounds require additional nondegeneracy assumptions on the learned or empirical velocity field. The flux-equivalent construction preserves marginal paths, and therefore cannot remove density-level memorization when the chosen empirical path itself ends at empirical atoms or a narrow kernel density estimator. 

{\bf Acknowledgment.} SHL would like to acknowledge support from the Wallenberg Initiative on Networks and Quantum Information   and the Swedish Research Council (VR/2021-03648).












\bibliographystyle{plain}
\bibliography{reference}

\newpage 
\appendix
\section*{Appendix}
This appendix is organized as follows. In App. \ref{app:relatedwork} we discuss related work. In App. \ref{app:proof} we provide detailed proof of the theoretical results presented in the main note. In App. \ref{app:numerics} we provide details of the numerical validation.

\section{Related Work} \label{app:relatedwork}


\noindent {\bf Flow Matching and related models.}
Flow Matching (FM) \cite{lipman2022flow, lipman2024flow} and Conditional Flow Matching (CFM) \cite{tong2023improving} have been developed as scalable alternatives \cite{esser2024scaling} to diffusion-based generative models \cite{song2020score, lai2025principles}. Recent work has analyzed their statistical, geometric, and algorithmic foundations, including distributional properties of FM~\cite{kunkel2025distribution}, particle and bridge-based interpretations~\cite{bamberger2025carr}, and geometric structure and gauge freedom in learned flow-based and diffusion models~\cite{wanelucidating, horvat2024gauge}. Extensions include guided generation~\cite{feng2025guidance}, statistical efficiency analyses~\cite{mena2025statistical}, rigorous comparisons between FM and optimal transport~\cite{hertrich2025relation, wald2025flow}, and related studies on spatio-temporal physical systems \cite{lim2024elucidating, erichson2025flex}. The kinetic behavior of flow-based samplers has also been examined in~\cite{shaul2023kinetic, li2025enfopath, li2026kinetic}.

\noindent {\bf Empirical FM, memorization, and density-estimation viewpoints.}
A growing body of work studies memorization, generalization, and interpolation phenomena in modern generative models. For diffusion models, prior work has analyzed identifiability, overfitting, and deterministic sampling behavior~\cite{pidstrigach2022score, yoon2023deterministic}. Further studies provide theoretical and empirical characterizations of interpolation, dataset coverage, and memorization tendencies~\cite{lyu2025resolving, scarvelis2023closed, baptista2025memorization, bertrand2025closed, chen2025interpolation}. For flow matching more specifically, recent work connects empirical FM to kernel density estimation and minimax nonparametric rates, making explicit that finite-sample FM can be understood as an implicit distribution estimator rather than only a transport learner \cite{kunkel2025minimax}. Our treatment complements this line by isolating the distinction between raw empirical target plug-in and smoothed plug-in targets, and by showing that the raw empirical minimizer generically develops non-gradient structure.

\noindent {\bf Conservativity, gauge freedom, and divergence alignment.}
Recent work has emphasized that the properties of vector field beyond pointwise
velocity matching can affect generative dynamics. Horvat and
Pfister~\cite{horvat2024gauge} study gauge freedom in diffusion models,
showing that vector fields need not be conservative to yield exact sampling or
density estimation when the non-conservative remainder satisfies an appropriate
gauge condition. In a complementary direction, \cite{huang2026improving} shows that conditional flow matching alone
does not necessarily control the learned probability path and propose aligning
both the flow and its divergence. Our work is related in spirit, but focuses on
a different finite-sample phenomenon: after replacing the target law by an
empirical or smoothed plug-in surrogate, the exact empirical FM minimizer and
its flux-equivalent representatives are analyzed directly. In particular,
flux-null vector fields preserve the prescribed empirical marginal path
while changing the particle-level dynamics.

\noindent {\bf Understanding and improving the sampling process.}
A complementary literature studies the dynamics and stability of generative sampling. This includes analyses of Lipschitz regularity and stability~\cite{chen2025lipschitz}, and methods aimed at accelerating or manipulating the generation process~\cite{gagneux2025generation, stancevic2025entropic, gupta2026sharpenflowsharpnessawaresampling}. For diffusion and score-based models, \cite{hurault2025score} examines how score estimation affects sampling quality. Our work adds to this view by characterizing the structural loss of gradient-field behavior and the concentration of kinetic energy induced by empirical FM.

\section{Proof of Theoretical Results} \label{app:proof}

\subsection{Proof of Proposition~\ref{prop1}}

\begin{proof}
Let $t \in [0,1]$ be given.
Let \(I\) be uniformly distributed on \(\{1,\ldots,N\}\), let
\(X=x^{(I)}\), and, conditional on \(X=x^{(i)}\), let
\(Z_t\sim p_t(\cdot\mid x^{(i)})\). For each \(i\), write
\(
    q_i(t,z):=p_t(z\mid x^{(i)}),
    \ 
    V_i(t,z):=v(t,z\mid x^{(i)}).
\)
For affine conditional flows,
\(
    V_i(t,z)=a_t(x^{(i)})z+b_t(x^{(i)}).
\)
The empirical marginal density at time \(t\) is
\(
    \hat p_t(z)=\frac1N\sum_{i=1}^N q_i(t,z).
\)

The empirical CFM objective can be written as
\[
\begin{aligned}
    \widehat{\mathcal L}_{\mathrm{CFM}}[v']
    &=
    \mathbb E_t\left[
    \frac1N\sum_{i=1}^N
    \int_{\mathbb R^d}
    \|v'(t,z)-V_i(t,z)\|^2 q_i(t,z)\,dz
    \right]  \\
    &=
    \mathbb E_t\left[
    \int_{\mathbb R^d}
    \frac1N\sum_{i=1}^N
    \|v'(t,z)-V_i(t,z)\|^2 q_i(t,z)\,dz
    \right].
\end{aligned}
\]
Define
\(
    w_i(t,z)
    :=
    \frac{q_i(t,z)}
    {\sum_{j=1}^N q_j(t,z)}.
\)
Then
\(
    \frac1N q_i(t,z)
    =
    \hat p_t(z)w_i(t,z),
\)
and therefore
\[
\begin{aligned}
    \widehat{\mathcal L}_{\mathrm{CFM}}[v']
    &=
    \mathbb E_t\left[
    \int_{\mathbb R^d}
    \sum_{i=1}^N
    w_i(t,z)\|v'(t,z)-V_i(t,z)\|^2
    \hat p_t(z)\,dz
    \right].
\end{aligned}
\]

For fixed \((t,z)\), consider the function of \(a\in\mathbb R^d\)
\[
    F_{t,z}(a)
    :=
    \sum_{i=1}^N w_i(t,z)\|a-V_i(t,z)\|^2.
\]
Since the weights are nonnegative and sum to one, completing the square gives
\[
    F_{t,z}(a)
    =
    \left\|a-\sum_{i=1}^N w_i(t,z)V_i(t,z)\right\|^2
    +
    \sum_{i=1}^N w_i(t,z)\|V_i(t,z)\|^2
    -
    \left\|\sum_{i=1}^N w_i(t,z)V_i(t,z)\right\|^2.
\]
Thus, for each \((t,z)\) with \(\hat p_t(z)>0\), the unique pointwise minimizer is
\[
    \hat v^\ast(t,z)
    =
    \sum_{i=1}^N w_i(t,z)V_i(t,z).
\]
Substituting the affine form of \(V_i\), we obtain
\[
    \hat v^\ast(t,z)
    =
    \sum_{i=1}^N
    w_i(t,z)\bigl(a_t(x^{(i)})z+b_t(x^{(i)})\bigr).
\]

Equivalently, this is the conditional expectation
\(
    \hat v^\ast(t,z)
    =
    \mathbb E[v(t,z\mid X)\mid Z_t=z],
\)
since Bayes' rule gives
\[
    \mathbb P(X=x^{(i)}\mid Z_t=z)
    =
    \frac{N^{-1}p_t(z\mid x^{(i)})}
    {N^{-1}\sum_{j=1}^N p_t(z\mid x^{(j)})}
    =
    w_i(t,z).
\]

It remains to justify uniqueness in the stated function space. The previous completion-of-squares identity yields
\[
\begin{aligned}
    \widehat{\mathcal L}_{\mathrm{CFM}}[v']
    &=
    \mathbb E_t\left[
    \int_{\mathbb R^d}
    \|v'(t,z)-\hat v^\ast(t,z)\|^2\hat p_t(z)\,dz
    \right]
    +
    C,
\end{aligned}
\]
where \(C\) is independent of \(v'\). Therefore \(v'\) minimizes the empirical CFM objective over
\(L^2(dt\,\hat p_t(dz);\mathbb R^d)\) if and only if
\(
    v'(t,z)=\hat v^\ast(t,z)
\)
for \(dt\otimes\hat p_t\)-almost every \((t,z)\). Hence the minimizer is unique as an element of
\(L^2(dt\,\hat p_t(dz);\mathbb R^d)\).

Finally, the empirical FM objective centered at the marginal velocity \(\hat v^\ast\) is
\[
    \widehat{\mathcal L}_{\mathrm{FM}}[v']
    =
    \mathbb E_t\left[
    \int_{\mathbb R^d}
    \|v'(t,z)-\hat v^\ast(t,z)\|^2\hat p_t(z)\,dz
    \right],
\]
so it has the same unique \(dt\otimes\hat p_t\)-a.e. minimizer. This completes the proof.
\end{proof}

\subsection{Proof of Proposition \ref{prop2}}
\begin{proof}[Proof of Proposition \ref{prop2}]
Fix \(t\in[0,T]\), with \(T<1\) in the unregularized rectified-flow case.  By the Poincar\'e lemma \cite{charles2022iterated}, a continuously differentiable vector field $F:\R^d\to\R^d$ on the simply connected domain $\R^d$ is a gradient field if and only if its Jacobian matrix $J_F$ is symmetric everywhere. Hence it suffices to characterize when the Jacobian of $\hat v^*(t,\cdot)$ is symmetric.

Write
\[
\hat v^*(t,z)=\sum_{i=1}^N w_i(t,z)v_i(t,z),
\qquad
v_i(t,z)=a_t(x^{(i)})z+b_t(x^{(i)}).
\]
Using the product rule $\nabla_z(cu)=cJ_u+u(\nabla_z c)^\top$ for a scalar-valued function $c$ and a vector-valued function $u$, we obtain
\[
J_{\hat v^*}(t,z)=\sum_{i=1}^N\Big(w_i(t,z)J_{v_i}(t,z)+v_i(t,z)(\nabla_z w_i(t,z))^\top\Big).
\]
Since $a_t(x^{(i)})$ is a scalar, the Jacobian of the affine field $v_i$ is
\(
J_{v_i}(t,z)=a_t(x^{(i)})I_d,
\)
which is symmetric. Therefore the only possible skew-symmetric contribution to $J_{\hat v^*}$ comes from the spatial variation of the weights $w_i(t,z)$. Taking the transpose and subtracting gives
\[
J_{\hat v^*}(t,z)-J_{\hat v^*}(t,z)^\top
=\sum_{i=1}^N\Big(v_i(t,z)(\nabla_z w_i(t,z))^\top-(\nabla_z w_i(t,z))v_i(t,z)^\top\Big).
\]
Consequently, $J_{\hat v^*}(t,z)$ is symmetric for all $z$ if and only if
$$\sum_{i=1}^N \left( v_i(t, z) \nabla_z w_i(t, z)^\top - \nabla_z w_i(t, z) v_i(t, z)^\top \right) = 0,$$
which is exactly the criterion stated in Proposition \ref{prop2}.
\end{proof}

\subsection{Proof of Proposition~\ref{prop:flux_equiv}}

\begin{proof}
By assumption, \(\hat v_t\) satisfies
\[
    \partial_t\hat p_t+\nabla\cdot(\hat p_t\hat v_t)=0
\]
in the weak, or distributional, sense. Since \(r_t\in\mathcal R_{\hat p_t}\)
for a.e. \(t\), we have
\[
    \nabla\cdot(\hat p_t r_t)=0
\]
in \(\mathcal D'(\mathbb R^d)\) for a.e. \(t\). Therefore, for
\(u_t=\hat v_t+r_t\),
\[
\begin{aligned}
    \partial_t\hat p_t+\nabla\cdot(\hat p_tu_t)
    &=
    \partial_t\hat p_t+\nabla\cdot\bigl(\hat p_t(\hat v_t+r_t)\bigr) =
    \partial_t\hat p_t+\nabla\cdot(\hat p_t\hat v_t)
    +
    \nabla\cdot(\hat p_t r_t) 
    =0
\end{aligned}
\]
in the distributional sense, for a.e. \(t\in[0,T]\). Hence \(u_t\) realizes
the same marginal density evolution as \(\hat v_t\), and so the same empirical marginal
path at the level of the continuity equation.

If, in addition, the ODE flows associated with \(\hat v_t\) and \(u_t\) are
well posed and the continuity equation is unique in the relevant solution
class, then any solution starting from \(\hat p_0\) and satisfying the above
continuity equation must coincide with \((\hat p_t)_{t\in[0,T]}\). Therefore
both flows push \(\hat p_0\) forward to \(\hat p_t\).
\end{proof}

\subsection{Proof of Proposition \ref{prop:gaussian_flux_null}}

\begin{proof}[Proof of Proposition~\ref{prop:gaussian_flux_null}]
Fix \(t\) and suppress the \(t\)-dependence. Write
\(p_i(z)=\mathcal N(z;m_i,\Sigma_i)\). Since
\[
    \hat p(z)w_i(z)=\frac1N p_i(z),
\]
we have
\[
    \hat p(z)r^A(z)
    =
    \frac1N
    \sum_{i=1}^N
    p_i(z)\Sigma_iA_i(z-m_i).
\]
It is enough to show that each component current has zero divergence. Fix
\(i\), and write \(m=m_i\), \(\Sigma=\Sigma_i\), \(A=A_i\), and \(y=z-m\).
Since
\(
    \nabla_z\log p_i(z)=-\Sigma^{-1}y,
\)
we obtain
\[
\begin{aligned}
    \nabla_z\cdot\{p_i(z)\Sigma A y\}
    &=
    p_i(z)\operatorname{tr}(\Sigma A)
    +
    p_i(z)(-\Sigma^{-1}y)^\top\Sigma A y \\
    &=
    p_i(z)\operatorname{tr}(\Sigma A)
    -
    p_i(z)y^\top A y.
\end{aligned}
\]
Because \(\Sigma\) is symmetric and \(A\) is antisymmetric,
\(\operatorname{tr}(\Sigma A)=0\). Also \(y^\top Ay=0\) for every
\(y\in\mathbb R^d\). Hence
\[
    \nabla_z\cdot\{p_i(z)\Sigma_iA_i(z-m_i)\}=0.
\]
Summing over \(i\) gives
\(
    \nabla\cdot(\hat p_t r_t^A)=0.
\)
Since the weights
are nonnegative and sum to one,
\[
    \|r_t^A(z)\|
    \le
    \sum_{i=1}^N
    w_i(t,z)\|\Sigma_i(t)A_i(t)\|_{\mathrm{op}}\|z-m_i(t)\|
    \le
    R_T^A(\|z\|+M_T).
\]
Since \(\hat p_t\) is a finite Gaussian mixture, it has finite second moment.
Therefore \(r_t^A\in L^2(\hat p_t;\mathbb R^d)\). Hence
\(r_t^A\in\mathcal R_{\hat p_t}\).

Together with the assumed \(L^2(\hat p;\mathbb R^d)\)-membership, this proves
\(r^A\in\mathcal R_{\hat p}\).
\end{proof}

\subsection{Proof of Proposition \ref{prop3}}

\begin{proof}[Proof of Proposition \ref{prop3} (a)]
Since $p_1 = \mathcal{N}(m_1, \Sigma_1)$, we have, for all $y \in \R^d$,
\begin{align}
- \log p_1(y) &= \frac{1}{2} (y - m_1)^T \Sigma_1^{-1} (y-m_1) + \frac{1}{2} \log \det (2 \pi \Sigma_1) \\
&= \frac{1}{2} y^T \Sigma_1^{-1} y - y^T \Sigma_1^{-1} m_1 + \frac{1}{2} m_1^T \Sigma_1^{-1} m_1 + \frac{1}{2} \log \det (2 \pi \Sigma_1).
\end{align}

Meanwhile, $E(y) = \|y-R^{-1}(y)\|^2 =  \|y- \Sigma_1^{-1/2}(y-m_1) \|^2 = \| (I_d - \Sigma_1^{-1/2})y + \Sigma_1^{-1/2} m_1 \|^2.$ Expanding the term and then regrouping the resulting terms, we obtain, for all $y \in \R^d$, 
$$\frac{1}{2} E(y) = \frac{1}{2} y^T(I-2 \Sigma_1^{-1/2} + \Sigma_1^{-1}) y + m_1^T \Sigma_1^{-1/2} y - m_1^T \Sigma_1^{-1} y + \frac{1}{2} m_1^T \Sigma_1^{-1} m_1.$$

The desired result then follows from the above formula for $-\log p_1(y)$ and $\frac{1}{2} E(y)$. 
\end{proof}

Before proving Proposition \ref{prop3} (b), we need the following auxiliary result. 
\begin{lemma} \label{lem_1}
    Let $W \sim \mathcal{N}(0, 1)$ be a scalar standard Gaussian random variable. Let $a, b \in \mathbb{R}$ be constants. If $b < \frac{1}{2}$, then:
    $$\mathbb{E}\left[e^{a W + bW^2}\right] = \frac{1}{\sqrt{1 - 2b}} \exp\left( \frac{a^2}{2(1 - 2b)} \right).$$
\end{lemma}
\begin{proof}
    We shall apply the integral formula: for $A > 0$, 
    $$\int_{-\infty}^{\infty} e^{-Ax^2 + Bx} \, dx = \sqrt{\frac{\pi}{A}} \exp\left(\frac{B^2}{4A}\right).$$
    
    The expectation is:
    $$\mathbb{E}[e^{aW + bW^2}] = \frac{1}{\sqrt{2\pi}} \int_{-\infty}^{\infty} \exp\left( -\left(\frac{1}{2} - b\right)w^2 + aw \right) \, dw.$$
    Identify $A = \frac{1}{2} - b$ (which is positive since $b < 1/2$) and $B = a$, and apply the formula  gives: 
    $$\begin{aligned}
\mathbb{E}\left[e^{a W + bW^2}\right] &= \frac{1}{\sqrt{2\pi}} \cdot \sqrt{\frac{\pi}{\frac{1}{2} - b}} \cdot \exp\left( \frac{a^2}{4(\frac{1}{2} - b)} \right) = \frac{1}{\sqrt{1 - 2b}} \exp\left( \frac{a^2}{2(1 - 2b)} \right).
\end{aligned}$$
\end{proof}

With this lemma in place, we can now prove part (b) in Proposition \ref{prop3}.  

\begin{proof}[Proof of Proposition \ref{prop3} (b)]
    The RF map is given as $R(x) = m_1 + \Sigma_1^{1/2}x$, and the inverse map is given by $R^{-1}(y) = \Sigma_1^{-1/2}(y - m_1)$. We analyze  the random variable $E(Y)$ where $Y \sim p_1$. Since $p_1$ is the pushforward of $p_0 = \mathcal{N}(0, I_d)$ through $R$, we can parameterize $Y$ using  $X \sim \mathcal{N}(0, I_d)$ via $Y = R(X)$.
    
    Substituting this into the energy definition, we have 
    $E(Y) = \| Y - R^{-1}(Y) \|^2.$
    Since $R^{-1}(R(X)) = X$ by definition of the inverse, this simplifies to $E = \| R(X) - X \|^2.$
    Using the definition of $R(X)$:$$E = \| (m_1 + \Sigma_1^{1/2}X) - X \|^2 = \| m_1 + (\Sigma_1^{1/2} - I_d)X \|^2.$$

    Let $A = \Sigma_1^{1/2} - I_d$. Note that $A$ is symmetric and we can consider the eigen-decomposition $A = U D U^T$, where $U$ is orthogonal and $D$ is diagonal with elements $d_i$. The eigenvalues of $\Sigma_1^{1/2}$ are $\sqrt{\lambda_i(\Sigma_1)}$. Thus, the eigenvalues of $A$ are 
    $d_i = \sqrt{\lambda_i(\Sigma_1)} - 1.$
    The kinetic energy can then be written as 
    $E = \| m_1 + U D U^T X \|^2.$ 
    
    Since the Euclidean norm is rotation-invariant, $\|v\|^2 = \|U^T v\|^2$ for any orthogonal matrix $U$, we obtain:
    $$E = \| U^T m_1 + D (U^T X) \|^2$$Let $\tilde{m} = U^T m_1$ (note $\|\tilde{m}\|^2 = \|m_1\|^2$) and $Z = U^T X$. Since $X \sim \mathcal{N}(0, I_d)$ and $U$ is orthogonal, $Z \sim \mathcal{N}(0, I_d)$. The energy decomposes into a sum of independent terms:
    $$E = \sum_{i=1}^d (\tilde{m}_i + d_i Z_i)^2.$$

    Let $u > 0$ be given. Applying the Chernoff bound gives 
    $\mathbb{P}(E \geq u) \leq e^{-tu} \mathbb{E}[e^{tE}]$ for any $t > 0$. Using the independence of $Z_i$, we have: 
    $$\mathbb{E}[e^{tE}] = \prod_{i=1}^d \mathbb{E}\left[ \exp\left( t(\tilde{m}_i + d_i Z_i)^2 \right) \right] =: \prod_{i=1}^d M_i.$$ 
    
    Expanding the term in the exponents, we see that:
    $$t(\tilde{m}_i^2 + 2\tilde{m}_i d_i Z_i + d_i^2 Z_i^2) = (t\tilde{m}_i^2) + (2t\tilde{m}_i d_i)Z_i + (t d_i^2)Z_i^2.$$
    
    Now, we apply Lemma \ref{lem_1} for $\mathbb{E}[e^{aW + bW^2}]$ with $W=Z_i$, $a = 2t\tilde{m}_i d_i$ and $b = t d_i^2$, for $b < 1/2$. 
    Let $\rho = \max_i (\sqrt{\lambda_i(\Sigma_1)} - 1)^2 = \max_i d_i^2$ (which is positive since we assume $\Sigma_1 \neq I_d$) and  choose $t = \frac{1}{4\rho}$. Then $b = \frac{d_i^2}{4\rho} \leq \frac{1}{4} < \frac{1}{2}$, and so the condition needed to apply the lemma is satisfied. 
    
    Applying the lemma to the $M_i$, we have: 
    $$M_i =  \frac{1}{\sqrt{1 - 2t d_i^2}} \cdot \exp\left( t\tilde{m}_i^2 + \frac{(2t\tilde{m}_i d_i)^2}{2(1 - 2t d_i^2)} \right).$$ 

    Now, we bound the terms:
    $2t d_i^2 = \frac{d_i^2}{2\rho} \leq \frac{1}{2}$. Thus $\sqrt{1 - 2t d_i^2} \geq \sqrt{1/2}$, and $\frac{1}{\sqrt{1 - 2t d_i^2}} \leq \sqrt{2}$. 
    Thus, for the term in the exponent of $M_i$:
    $$t\tilde{m}_i^2 + \frac{4t^2 \tilde{m}_i^2 d_i^2}{2(1 - 2t d_i^2)} = t\tilde{m}_i^2 \left( 1 + \frac{2t d_i^2}{1 - 2t d_i^2} \right) = \frac{t\tilde{m}_i^2}{1 - 2t d_i^2}.$$ 
    Since $1 - 2t d_i^2 \geq 1/2$, 
    $$\frac{t\tilde{m}_i^2}{1 - 2t d_i^2} \leq 2t\tilde{m}_i^2 = \frac{\tilde{m}_i^2}{2\rho}.$$
    
    Combining these, we have:
    $$M_i \leq \sqrt{2} \exp\left( \frac{\tilde{m}_i^2}{2\rho} \right).$$ 
    Therefore, 
    $$\mathbb{E}[e^{tE}] \leq \prod_{i=1}^d \left( \sqrt{2} e^{\frac{\tilde{m}_i^2}{2\rho}} \right) = 2^{d/2} \exp\left( \frac{\sum \tilde{m}_i^2}{2\rho} \right) = 2^{d/2} \exp\left( \frac{\|m_1\|^2}{2\rho} \right) =: C.$$
    
    Finally, substituting this into the earlier Chernoff bound: $$\mathbb{P}(E \geq u) \leq e^{-tu} C = C \exp\left( -\frac{u}{4\rho} \right).$$
    
\end{proof}

\subsection{Proof of Theorem \ref{thm1}}
Before that, we need the following lemma.

\begin{lemma} \label{app_lem2}
    Let $X_0 \sim \mathcal{N}(0, I_d)$. Define $U = \frac{\|X_0\|^2}{d}.$ 
    For all $s \geq 2$, we have:
    $$\mathbb{P}(U \geq s) \leq \exp\left(-\frac{sd}{16} \right).$$
\end{lemma}
\begin{proof}
    First, we claim that for all $s \geq 1$, 
    \begin{equation} \label{eq_claim}
        \mathbb{P}\left(U \geq s \right) \leq \exp\left( -\frac{d}{2} f(s)  \right),
    \end{equation}
    where $f(s) = s - 1 - \ln(s)$.

To verify this claim, let $S := \|X_0\|^2$ and compute, for $\lambda > 0$, 
\begin{align}
    \mathbb{P}(S \geq ds) &=  \mathbb{P}(e^{\lambda S} \geq e^{\lambda d s}) \\
    &\leq e^{-\lambda d s} \mathbb{E}[e^{\lambda S}] = \frac{ e^{-\lambda d s}}{(1-2\lambda)^{d/2}},
\end{align}
where we have used the fact that $\|X_0\|^2 \sim \chi_d^2$ (chi-squared distributed) and the formula for its moment generating function in the last line. 
Choosing $\lambda = \frac{s-1}{2s} \in (0,1/2)$ minimizes the upper bound. Plugging this minimizer back into the upper bound, we obtain the result as claimed.

Now, observe that for $s \geq 2$, $f(s) \geq s/8$. Therefore, using \eqref{eq_claim} and this observation, we have, for all $s \geq 2$,
 \begin{equation}
        \mathbb{P}\left(U \geq s \right) \leq  \exp\left(-\frac{sd}{16} \right),
    \end{equation}
which is the result that we wanted to show.
\end{proof}

With this lemma in place, we can now prove Theorem \ref{thm1}. 

\begin{proof}[Proof of Theorem \ref{thm1}]
Let $T \in [0,1)$ and $\mathcal{D}_N$ be given. For all $t \in [0,T]$ and $z \in \R^d$, 
\begin{align}
    \|\hat{v}^*(t,z)\| &\leq \frac{1}{1-t} \sum_{i=1}^N w_i(t,z) \|x^{(i)} - z\| \\
    &\leq \frac{1}{1-t} \sum_{i=1}^N w_i(t,z) (\|x^{(i)}\| + \|z\|) \\
    &\leq \frac{1}{1-t}(M + \|z\|), \label{eq_bb}
\end{align}
where we have used the fact that $\sum_i w_i(t,z) = 1$ and the notation $M := \max_i \|x^{(i)}\|$.

Let $r_t := \|\psi_t(X_0)\|$. For all $t$ with $r_t > 0$, 
$$\dot{r}_t := \frac{dr_t}{dt} = \frac{\psi_t(X_0) \cdot \dot{\psi}_t(X_0)}{\|\psi_t(X_0)\|} \leq \frac{|\psi_t(X_0) \cdot \dot{\psi}_t(X_0)|}{\|\psi_t(X_0)\|} \leq \|\dot{\psi}_t(X_0)\| = \|\hat{v}^*(t, \psi_t(X_0))\|, $$
where we have used the chain rule for differentiation and Cauchy-Schwarz inequality.

Then, using \eqref{eq_bb}:
$$\dot{r}_t \leq \frac{1}{1-t} (M + r_t) $$
and so $(1-t) \dot{r}_t - r_t \leq M$.
Now, $$\frac{d}{dt}\left((1-t) r_t \right) = (1-t) \dot{r}_t - r_t \leq M.$$
Integrating both sides from $0$ to $t$ gives (and noting that $r_0 = \|X_0\|$): 
\begin{align} 
    (1-t) r_t - r_0 &\leq M t \\
    (1-t) r_t &\leq \|X_0\| + Mt \\
    \|\psi_t(X_0)\| &\leq \frac{\|X_0\| + Mt}{1-t} =: c_1(t) \|X_0\| + c_2(t) M, \label{eq_bb2}
\end{align}
where $c_1(t) = 1/(1-t)$ and $c_2(t) = t/(1-t)$.

Let $\hat{V}_t := \hat{v}^*(t, \psi_t(X_0))$. Using \eqref{eq_bb} and \eqref{eq_bb2}, we have:
\begin{align}
    \|\hat{V}_t\| &\leq \frac{1}{1-t}(M + \|\psi_t(X_0)\|) \leq \frac{1}{1-t}(M + c_1(t) \|X_0\| + c_2(t) M) \leq c_1^2(t) (M + \|X_0\|).
\end{align}

Therefore, 
\begin{align}
    K_t := \|\hat{V}_t\|^2 &\leq c_1^4(t) (\|X_0\| + M)^2 \leq 2 c_1^4(t) (\|X_0\|^2 + M^2),
\end{align}
where we have used the inequality $(x+y)^2 \leq 2(x^2 + y^2)$ for $x, y \in \R$. 

Integrating from $0$ to $T$ on both sides gives:
$$E_T = \int_0^T K_t dt \leq c_3(T) (\|X_0\|^2 + M^2), $$ where $c_3(T) =2  \int_0^T c_1^4(t) dt = \frac{2}{3}((1-T)^{-3} - 1)$.

Now, for any $u>0$, since $\{K_t \geq u\} \subset \left\{\|X_0\|^2 \geq \frac{u}{2c_1^4(t)} - M^2 \right\}$, we have:
\begin{align}
    \mathbb{P}[K_t \geq u \mid \mathcal{D}_N] &\leq \mathbb{P}[\|X_0\|^2/d \geq s \mid \mathcal{D}_N], 
\end{align}
where $s := u/(2dc_1^4(t)) - M^2/d$.

Fix $U_t := 2c_1^4(t)(2d+M^2)$, so that for every $u\ge U_t$ we have $s \ge 2$. Applying Lemma \ref{app_lem2} then gives
\[
\mathbb{P}[K_t \ge u \mid \mathcal{D}_N] \le \exp\!\left(-\frac{d}{16}\Big(\frac{u}{2dc_1^4(t)}-\frac{M^2}{d}\Big)\right)
= e^{M^2/16}\exp\!\left(-\frac{u}{32c_1^4(t)}\right).
\]
Thus part (a) holds with $C_t=e^{M^2/16}$ and $c_t=(1-t)^4/32$.

For part (b), define $U_T := c_3(T)(2d+M^2)$. Then, for every $u\ge U_T$, the same argument gives
\[
\mathbb{P}[E_T \ge u \mid \mathcal{D}_N] \le e^{M^2/16}\exp\!\left(-\frac{u}{16c_3(T)}\right).
\]
Hence part (b) holds with $C_T=e^{M^2/16}$ and $c_T = 1/(16 c_3(T)) = \frac{3}{32((1-T)^{-3} - 1)}$.
\end{proof}

\subsection{Proof of Theorem \ref{thm2_polydecay}}

\begin{proof}[Proof of Theorem~\ref{thm2_polydecay}]
The proof is analogous to that of Theorem~\ref{thm1}, with the Gaussian
tail bound replaced by the assumed power-law tail.

Recall from Proposition~\ref{prop1} that the empirical affine-flow minimizer
has the form
\[
\hat v^\ast(t,z)
= \sum_{i=1}^N w_i(t,z)\,\big(a_t(x^{(i)}) z + b_t(x^{(i)})\big),
\]
where the weights \(w_i(t,z)\) are nonnegative and sum to one. By the definition of
\[
A_{\max} := \sup_{t\in[0,T],\,i\in[N]} |a_t(x^{(i)})|,
\qquad
B_{\max} := \sup_{t\in[0,T],\,i\in[N]} \|b_t(x^{(i)})\|,
\]
we have, for all \(t \in [0,T]\) and all \(z \in \mathbb{R}^d\),
\begin{equation}
\label{eq:linear-growth-vhat}
\|\hat v^\ast(t,z)\|
\le \sum_{i=1}^N w_i(t,z)\big(|a_t(x^{(i)})|\,\|z\| + \|b_t(x^{(i)})\|\big)
\le A_{\max}\|z\| + B_{\max}.
\end{equation}

Let \(\psi_t\) denote the flow driven by \(\hat v^\ast\), i.e.,
\[
\dot\psi_t(X_0) = \hat v^\ast(t,\psi_t(X_0)), 
\qquad \psi_0(X_0) = X_0,
\]
and define \(r_t := \|\psi_t(X_0)\|\). Whenever\footnote{The same differential inequality holds for the upper Dini derivative of \(r_t\), which is sufficient for Grönwall.} \(r_t > 0\), we have, by the chain rule and
Cauchy--Schwarz inequality,
\[
\dot r_t
= \frac{\psi_t(X_0)}{\|\psi_t(X_0)\|} \cdot \dot\psi_t(X_0)
\le \|\hat v^\ast(t,\psi_t(X_0))\|.
\]
Using \eqref{eq:linear-growth-vhat} at \(z = \psi_t(X_0)\) gives
\[
\dot r_t
\le A_{\max} r_t + B_{\max}.
\]
By Grönwall’s lemma, there exist constants \(C_1(T), C_2(T) > 0\), depending only on
\(T, A_{\max}, B_{\max}\), such that for all \(t \in [0,T]\),
\begin{equation}
\label{eq:rt-bound}
r_t = \|\psi_t(X_0)\|
\le C_1(T)\,\|X_0\| + C_2(T).
\end{equation}

Define \(V_t := \hat v^\ast(t,\psi_t(X_0))\) and the instantaneous kinetic energy
\(K_t := \|V_t\|^2\). Combining \eqref{eq:linear-growth-vhat} and \eqref{eq:rt-bound},
we obtain
\[
\|V_t\|
\le A_{\max} r_t + B_{\max}
\le A_{\max}\big(C_1(T)\|X_0\| + C_2(T)\big) + B_{\max}
\le C_3(T)\,\|X_0\| + C_4(T),
\]
for suitable constants \(C_3(T), C_4(T) > 0\) depending only on \(T, A_{\max}, B_{\max}\).
Hence, by the inequality \((x+y)^2 \le 2(x^2 + y^2)\),
\begin{equation}
\label{eq:Kt-X0}
K_t
= \|V_t\|^2
\le 2C_3(T)^2\|X_0\|^2 + 2C_4(T)^2
\le C_K(T)\,\big(\|X_0\|^2 + 1\big),
\end{equation}
where we may take \(C_K(T) := 2\max\{C_3(T)^2, C_4(T)^2\}\).
Integrating \eqref{eq:Kt-X0} over \(t \in [0,T]\) yields the same type of bound
for the integrated kinetic energy
$E_T := \int_0^T K_t\,dt,$
i.e.,
\begin{equation}
\label{eq:ET-X0}
E_T
\le C_E(T)\,\big(\|X_0\|^2 + 1\big),
\end{equation}
for some constant \(C_E(T) := T C_K(T) > 0\) depending only on \(T, A_{\max}, B_{\max}\).

\medskip\noindent
\emph{Tail bounds.}
From \eqref{eq:Kt-X0}, for any \(u > 0\),
\[
\{K_t \ge u\}
\subseteq \Bigl\{\|X_0\|^2 \ge \frac{u}{C_K(T)} - 1\Bigr\}.
\]
Fix \(U_t\) large enough so that for all \(u \ge U_t\),
\(\frac{u}{C_K(T)} - 1 \ge 1\). Writing \(s := \sqrt{\frac{u}{C_K(T)} - 1}\),
we obtain
\[
\mathbb{P}\bigl(K_t \ge u \,\big|\, D_N\bigr)
\le \mathbb{P}\bigl(\|X_0\| \ge s \,\big|\, D_N\bigr)
= \mathbb{P}\bigl(\|X_0\| \ge s\bigr),
\]
since \(X_0\) is independent of \(D_N\). By the heavy-tailed assumption on \(p_0\),
for all \(s \ge 1\),
\[
\mathbb{P}\bigl(\|X_0\| \ge s\bigr)
\le \frac{C_\alpha}{s^\alpha}.
\]
For $u \ge U_t$ large enough so that $s^2 = \frac{u}{C_K(T)} - 1 \ge \frac{u}{2C_K(T)}$,
we have:
\[
\frac{1}{s^\alpha} \le \left(\frac{2C_K(T)}{u}\right)^{\alpha/2},
\]
and hence
\[
\mathbb{P}\bigl(K_t \ge u \,\big|\, D_N\bigr)
\le \frac{C_t}{u^{\alpha/2}},
\]
for all sufficiently large $u$, for a constant $C_t>0$ depending only on $t$, $T$, $A_{\max}$, $B_{\max}$, $\alpha$, and $C_\alpha$. This proves the first inequality in Theorem \ref{thm2_polydecay}.

The argument for $E_T$ is identical, using \eqref{eq:ET-X0} in place of
\eqref{eq:Kt-X0}. For any $u > 0$,
\[
\{E_T \ge u\}
\subseteq \Bigl\{\|X_0\|^2 \ge \frac{u}{C_E(T)} - 1\Bigr\},
\]
and the same substitution $s = \sqrt{\frac{u}{C_E(T)} - 1}$ together with the heavy-tailed bound on $\|X_0\|$ yields
\[
\mathbb{P}(E_T \ge u \,\big|\, D_N)
\le \frac{C_T}{u^{\alpha/2}},
\]
for all sufficiently large $u$, for a constant $C_T>0$ depending only on $T$, $A_{\max}$, $B_{\max}$, $\alpha$, and $C_\alpha$. This proves the second inequality in Theorem \ref{thm2_polydecay}.

\end{proof}

\subsection{Proof of Proposition \ref{lem:linear_growth_tail}}

\begin{proof}
Let \(R_t:=\|X_t\|\). Since \(X_t\) is an absolutely continuous solution of
\(\dot X_t=u_t(X_t)\), the map \(t\mapsto R_t\) is absolutely continuous.
For a.e. \(t\) such that \(X_t\neq 0\), the chain rule gives
\[
    \frac{d}{dt}R_t
    =
    \frac{X_t}{\|X_t\|}\cdot \dot X_t
    =
    \frac{X_t}{\|X_t\|}\cdot u_t(X_t)
    \le
    \|u_t(X_t)\|.
\]
At times where \(X_t=0\), the same inequality holds for the a.e. derivative
by the standard inequality for the norm of an absolutely continuous curve.
Hence, for a.e. \(t\in[0,T]\),
\[
    \frac{d}{dt}R_t
    \le
    \|u_t(X_t)\|
    \le
    L_T\|X_t\|+B_T
    =
    L_T R_t+B_T.
\]
By Gr\"onwall's inequality,
\[
    R_t
    \le
    e^{L_Tt}R_0
    +
    B_T\int_0^t e^{L_T(t-s)}\,ds.
\]
Since \(R_0=\|X_0\|\), there exists a constant \(C_1=C_1(T,L_T,B_T)\) such that
\[
    R_t\le C_1(1+\|X_0\|),
    \qquad t\in[0,T].
\]
Using the linear-growth condition once more,
\[
    \|u_t(X_t)\|
    \le
    L_T R_t+B_T
    \le
    L_TC_1(1+\|X_0\|)+B_T.
\]
Thus there exists \(C_2=C_2(T,L_T,B_T)\) such that
\(
    \|u_t(X_t)\|
    \le
    C_2(1+\|X_0\|),
    \  t\in[0,T].
\)
Therefore
\[
    K_t^u
    =
    \|u_t(X_t)\|^2
    \le
    C_2^2(1+\|X_0\|)^2
    \le
    C_3(1+\|X_0\|^2),
\]
where \(C_3=C_3(T,L_T,B_T)\). Consequently,
\[
    E_T^u
    =
    \int_0^T K_t^u\,dt
    \le
    TC_3(1+\|X_0\|^2).
\]
After absorbing \(T\) into the constant, there exists
\(C_T=C_T(T,L_T,B_T)\) such that, for all \(t\in[0,T]\),
\[
    K_t^u\le C_T(1+\|X_0\|^2),
    \qquad
    E_T^u\le C_T(1+\|X_0\|^2).
\]

We now derive the tail bounds. First suppose \(X_0\sim\mathcal N(0,I_d)\).
Then \(\|X_0\|^2\sim\chi_d^2\), and hence there exist constants
\(c_d,C_d>0\) such that, for all sufficiently large \(s\),
\(\mathbb P(\|X_0\|^2\ge s)\le C_d e^{-c_d s}\). Therefore, for sufficiently
large \(\lambda\),
\[
\begin{aligned}
    \mathbb P(K_t^u\ge \lambda)
    &\le
    \mathbb P\left(C_T(1+\|X_0\|^2)\ge \lambda\right) =
    \mathbb P\left(\|X_0\|^2\ge \frac{\lambda}{C_T}-1\right) \le
    C e^{-c\lambda},
\end{aligned}
\]
for constants \(c,C>0\) depending only on \(T,L_T,B_T\) and \(d\). The same
argument, using \(E_T^u\le C_T(1+\|X_0\|^2)\), gives
\(\mathbb P(E_T^u\ge \lambda)\le C e^{-c\lambda}\) for all sufficiently large
\(\lambda\).

Now suppose instead that
\(\mathbb P(\|X_0\|\ge s)\le C_\alpha s^{-\alpha}\) for all \(s\ge 1\).
Using \(K_t^u\le C_T(1+\|X_0\|^2)\), we have, for sufficiently large
\(\lambda\),
\[
\begin{aligned}
    \mathbb P(K_t^u\ge \lambda)
    &\le
    \mathbb P\left(C_T(1+\|X_0\|^2)\ge \lambda\right) =
    \mathbb P\left(\|X_0\|\ge \sqrt{\frac{\lambda}{C_T}-1}\right).
\end{aligned}
\]
For sufficiently large \(\lambda\), the threshold
\(\sqrt{\lambda/C_T-1}\) is at least \(1\). Hence the polynomial tail
assumption gives
\[
    \mathbb P(K_t^u\ge \lambda)
    \le
    C_\alpha
    \left(\sqrt{\frac{\lambda}{C_T}-1}\right)^{-\alpha}.
\]
For large enough \(\lambda\), there exists \(c_T>0\) such that
\(\sqrt{\lambda/C_T-1}\ge c_T\sqrt{\lambda}\). Therefore,
\(
    \mathbb P(K_t^u\ge \lambda)
    \le
    C\lambda^{-\alpha/2}.
\)
The same argument applies to \(E_T^u\), since
\(E_T^u\le C_T(1+\|X_0\|^2)\). This proves the claimed polynomial upper-tail
bounds.
\end{proof}

\subsection{Proof of Theorem~\ref{thm:flux_equiv_affine_tail}}

\begin{proof}
Fix \(t\in[0,T]\). We first show that the probability current generated by
\(r_t^A\) has zero divergence. Write
\(p_i(t,z)=\mathcal N(z;m_i(t),\Sigma_i(t))\). Then
\[
    \hat p_t(z)r_t^A(z)
    =
    \frac1N
    \sum_{i=1}^N
    p_i(t,z)\Sigma_i(t)A_i(t)(z-m_i(t)).
\]
It is enough to check that each component current has zero divergence. Fix
\(i\), and write \(m=m_i(t)\), \(\Sigma=\Sigma_i(t)\), \(A=A_i(t)\), and
\(y=z-m\). Since \(\nabla_z\log p_i(t,z)=-\Sigma^{-1}y\), we have
\[
\begin{aligned}
    \nabla_z\cdot\left[p_i(t,z)\Sigma A y\right]
    &=
    p_i(t,z)\operatorname{tr}(\Sigma A)
    +
    p_i(t,z)(-\Sigma^{-1}y)^\top\Sigma A y \\
    &=
    p_i(t,z)\operatorname{tr}(\Sigma A)
    -
    p_i(t,z)y^\top A y.
\end{aligned}
\]
Because \(\Sigma\) is symmetric and \(A\) is antisymmetric,
\(\operatorname{tr}(\Sigma A)=0\) and \(y^\top A y=0\). Therefore
\[
    \nabla_z\cdot
    \left[
        p_i(t,z)\Sigma_i(t)A_i(t)(z-m_i(t))
    \right]
    =
    0.
\]
Summing over \(i\) gives
\(
    \nabla\cdot(\hat p_t r_t^A)=0.
\)

We next verify the required \(L^2(\hat p_t;\mathbb R^d)\)-membership. Since
the weights \(w_i(t,z)\) are nonnegative and sum to one,
\[
\begin{aligned}
    \|r_t^A(z)\|
    &\le
    \sum_{i=1}^N
    w_i(t,z)\|\Sigma_i(t)A_i(t)\|_{\mathrm{op}}\|z-m_i(t)\| \le
    R_T^A(\|z\|+M_T).
\end{aligned}
\]
Hence
\[
    \|r_t^A(z)\|^2
    \le
    2(R_T^A)^2(\|z\|^2+M_T^2).
\]
Since \(\hat p_t\) is a finite Gaussian mixture, it has finite second moment:
\[
    \int_{\mathbb R^d}\|z\|^2\hat p_t(z)\,dz
    =
    \frac1N\sum_{i=1}^N
    \left(\|m_i(t)\|^2+\operatorname{tr}\Sigma_i(t)\right)
    <\infty.
\]
Therefore \(r_t^A\in L^2(\hat p_t;\mathbb R^d)\). Together with
\(\nabla\cdot(\hat p_t r_t^A)=0\), this gives
\(r_t^A\in\mathcal R_{\hat p_t}\). Hence \(u_t^A=\hat v_t+r_t^A\) is
flux-equivalent to \(\hat v_t\). Since
\[
    \partial_t\hat p_t+\nabla\cdot(\hat p_t\hat v_t)=0,
\]
we also have
\(
    \partial_t\hat p_t+\nabla\cdot(\hat p_tu_t^A)=0.
\)

It remains to prove the growth bound for \(u_t^A\). Since the weights are
nonnegative and sum to one,
\[
\begin{aligned}
    \|\hat v_t(z)\|
    &\le
    \sum_{i=1}^N w_i(t,z)\|B_i(t)z+b_i(t)\| \le
    B_T^{\rm aff}\|z\|+b_T^{\rm aff}.
\end{aligned}
\]
Combining this with the bound on \(r_t^A\) above gives
\[
    \|u_t^A(z)\|
    \le
    (B_T^{\rm aff}+R_T^A)\|z\|
    +
    (b_T^{\rm aff}+R_T^A M_T).
\]
Thus \(u_t^A\) satisfies
\[
    \|u_t^A(z)\|
    \le
    L_T^A\|z\|+B_T^A,
    \qquad
    L_T^A:=B_T^{\rm aff}+R_T^A,
    \quad
    B_T^A:=b_T^{\rm aff}+R_T^A M_T.
\]
Applying Proposition~\ref{lem:linear_growth_tail} with \(L_T=L_T^A\) and
\(B_T=B_T^A\) gives the deterministic energy bounds and the stated Gaussian
or polynomial source-tail upper bounds.
\end{proof}

\subsection{Proof of Corollary \ref{cor:flux_equiv_vfrf}}
\begin{proof}
This is the special case of Theorem~\ref{thm:flux_equiv_affine_tail} with
\(m_i(t)=tx^{(i)}\), \(\Sigma_i(t)=\sigma_t^2I\), and
\(\sigma_t=1-(1-\sigma_{\min})t\). Then
\[
    \Sigma_i(t)A_i(t)(z-m_i(t))
    =
    \sigma_t^2A_i(t)(z-tx^{(i)}),
\]
which gives the stated flux-null remainder. Since
\(\sigma_t\in[\sigma_{\min},1]\), we have
\[
    \|\hat v_t(z)\|
    \le
    \frac{1-\sigma_{\min}}{\sigma_{\min}}\|z\|
    +
    \frac{M}{\sigma_{\min}}.
\]
Moreover,
\[
    \|r_t^A(z)\|
    \le
    \sigma_t^2
    \sum_{i=1}^N w_i(t,z)\|A_i(t)\|_{\mathrm{op}}\|z-tx^{(i)}\|
    \le
    A_{\max}(\|z\|+M).
\]
Combining these two inequalities gives the stated values of \(L_T^A\) and
\(B_T^A\). The flux-equivalence and source-tail conclusions then follow from
Theorem~\ref{thm:flux_equiv_affine_tail}.
\end{proof}


\section{Details on Empirical Validations}
\label{app:numerics}

This appendix gives implementation details for the numerical experiments in
Section~\ref{sec:numerics}. All experiments evaluate the closed-form empirical velocity directly; no neural network is trained.

\begin{figure}[!b]
\centering
\includegraphics[width=1\textwidth]{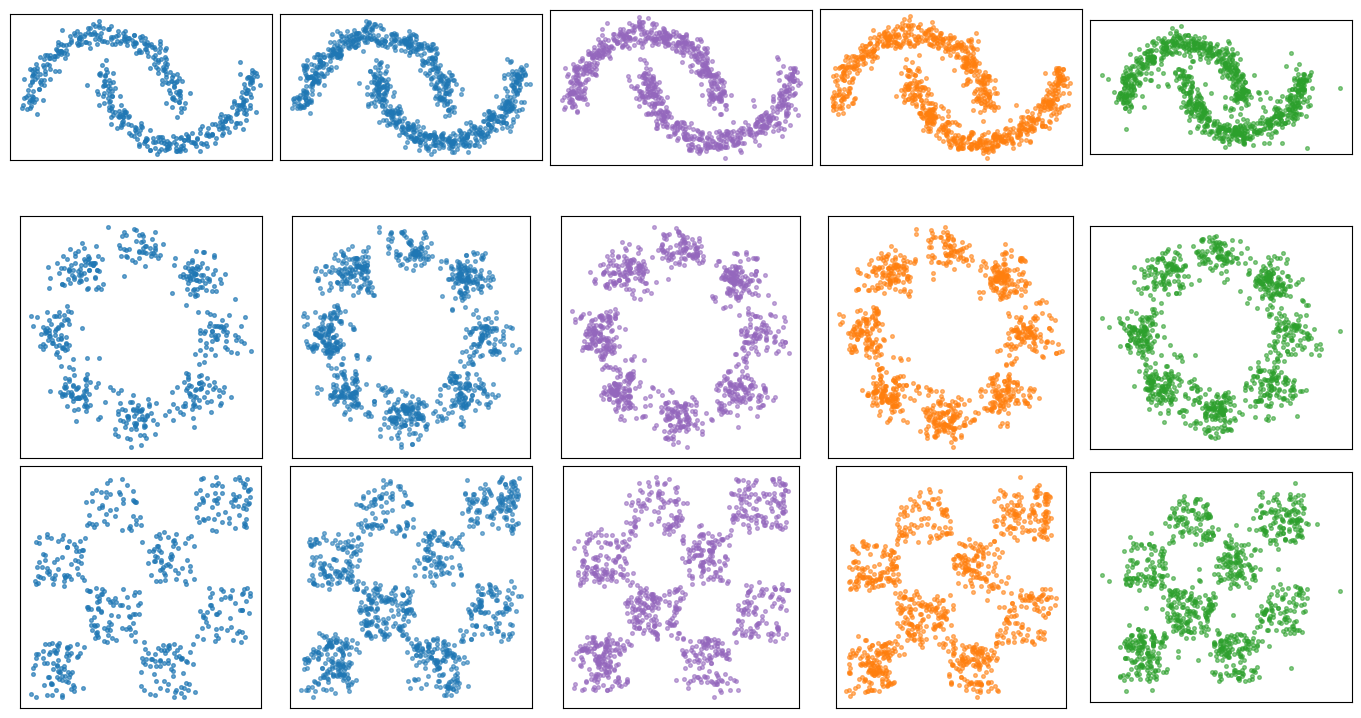}
\caption{
Visualization of generated samples for Two Moons, 8 Gaussians, and Checkerboard. They are near-terminal generated samples at \(T=0.97\). These plots are included only as a sanity check; the experiments are designed to study kinetic energy tails, not sample quality.
}
\label{fig:visual}
\end{figure}

\subsection{Empirical Affine-Flow Experiment}

For the empirical affine-flow experiment, we use the regularized affine path
\(
    Z_t=t x^{(i)} + s_t X_0,
    \ 
    s_t=1-(1-\sigma_{\min})t.
\)
The conditional density is
\(
    p_t(z\mid x^{(i)})
    =
    s_t^{-d}
    K\!\left(\frac{z-tx^{(i)}}{s_t}\right),
\)
and the conditional velocity is
\(
    v_i(t,z)
    =
    \frac{x^{(i)}-(1-\sigma_{\min})z}{s_t}.
\)
The empirical minimizer is therefore
\(
    \hat v(t,z)
    =
    \sum_{i=1}^N w_i(t,z)
    \frac{x^{(i)}-(1-\sigma_{\min})z}{s_t},
\)
where
\(
    w_i(t,z)
    =
    \frac{
    K\!\left((z-tx^{(i)})/s_t\right)
    }{
    \sum_{j=1}^N
    K\!\left((z-tx^{(j)})/s_t\right)
    }.
\)

The empirical sampler is integrated using forward Euler:
\(
    Z_{k+1}
    =
    Z_k+\Delta t\,\hat v(t_k,Z_k).
\)
The integrated kinetic energy is approximated by the left-endpoint rule
\(
    E_T^{\Delta t}
    =
    \sum_{k=0}^{n_{\mathrm{steps}}-1}
    \Delta t\,
    \|\hat v(t_k,Z_k)\|^2.
\)
This left-endpoint approximation is paired with the forward Euler trajectory. In contrast, the
affine sharpness experiment below uses a trapezoidal rule because the velocity can be evaluated
from a closed-form expression.

The empirical experiment settings are shown in Table~\ref{tab:empirical_numerics_settings}. The generated samples are shown in Figure \ref{fig:visual}.

\begin{table}[!h]
\centering
\small 
\begin{tabular}{ll}
\toprule
Parameter & Value \\
\midrule
Training samples \(N\) & \(500\) \\
Generated samples per seed \(M\) & \(1000\) \\
Number of seeds & \(5\) \\
Datasets & Two moons, eight Gaussians, checkerboard \\
Dimension & \(d=2\) \\
Sources & Gaussian, Student-\(t_2\), Student-\(t_5\), Student-\(t_{10}\) \\
Regularization & \(\sigma_{\min}=0.02\) \\
Integration horizon & \(T=0.97\) \\
Euler steps & \(100\) \\
Instantaneous energy time & \(t_{\mathrm{mid}}\approx0.55T\) \\
\bottomrule
\end{tabular}
\caption{Numerical settings for the empirical affine-flow experiments.}
\label{tab:empirical_numerics_settings}
\end{table}

For coordinate-wise Student-\(t_\nu\) sources in fixed dimension,
\(
    \mathbb P(\|X_0\|>s)\asymp s^{-\nu}.
\)
Since energy is often comparable to \(\|X_0\|^2\) in nondegenerate affine settings, the natural
benchmark for energy tails is
\(
    \mathbb P(E_T>u)\approx u^{-\nu/2}.
\)
For the nonlinear empirical affine-flow sampler, however, Theorem~\ref{thm2_polydecay} gives
only an upper-tail bound, not an exact tail-index identity. Therefore fitted log-log slopes for
the empirical sampler should be interpreted as qualitative diagnostics only.

\subsection{Diagnostics}

For each run, we compute the empirical survival function
\(
    \widehat S_E(u)
    =
    \frac1M
    \sum_{m=1}^M
    \mathbf 1\{E_T^{(m)}>u\}.
\)
We visualize \(\widehat S_E\) on both log-linear and log-log axes. Log-linear plots highlight
exponential-type behavior,
\(
    \log\widehat S_E(u)\approx a-cu,
\)
whereas log-log plots highlight polynomial-type behavior,
\(
    \log\widehat S_E(u)\approx a-\beta\log u.
\)
We also compute high-energy quantiles, including the empirical \(90\%\), \(95\%\), and \(99\%\)
quantiles of \(E_T\).

For further diagnostics, we also record the instantaneous kinetic energy
\(
    K_{t_{\mathrm{mid}}}
    =
    \|\hat v(t_{\mathrm{mid}},Z_{t_{\mathrm{mid}}})\|^2.
\)
Figure~\ref{fig:empirical_Kmid_survival_appendix} shows that the source-driven tail ordering for
\(K_{t_{\mathrm{mid}}}\) matches the behavior observed for \(E_T\).

\begin{figure}[!h]
\centering
\includegraphics[width=1\textwidth]{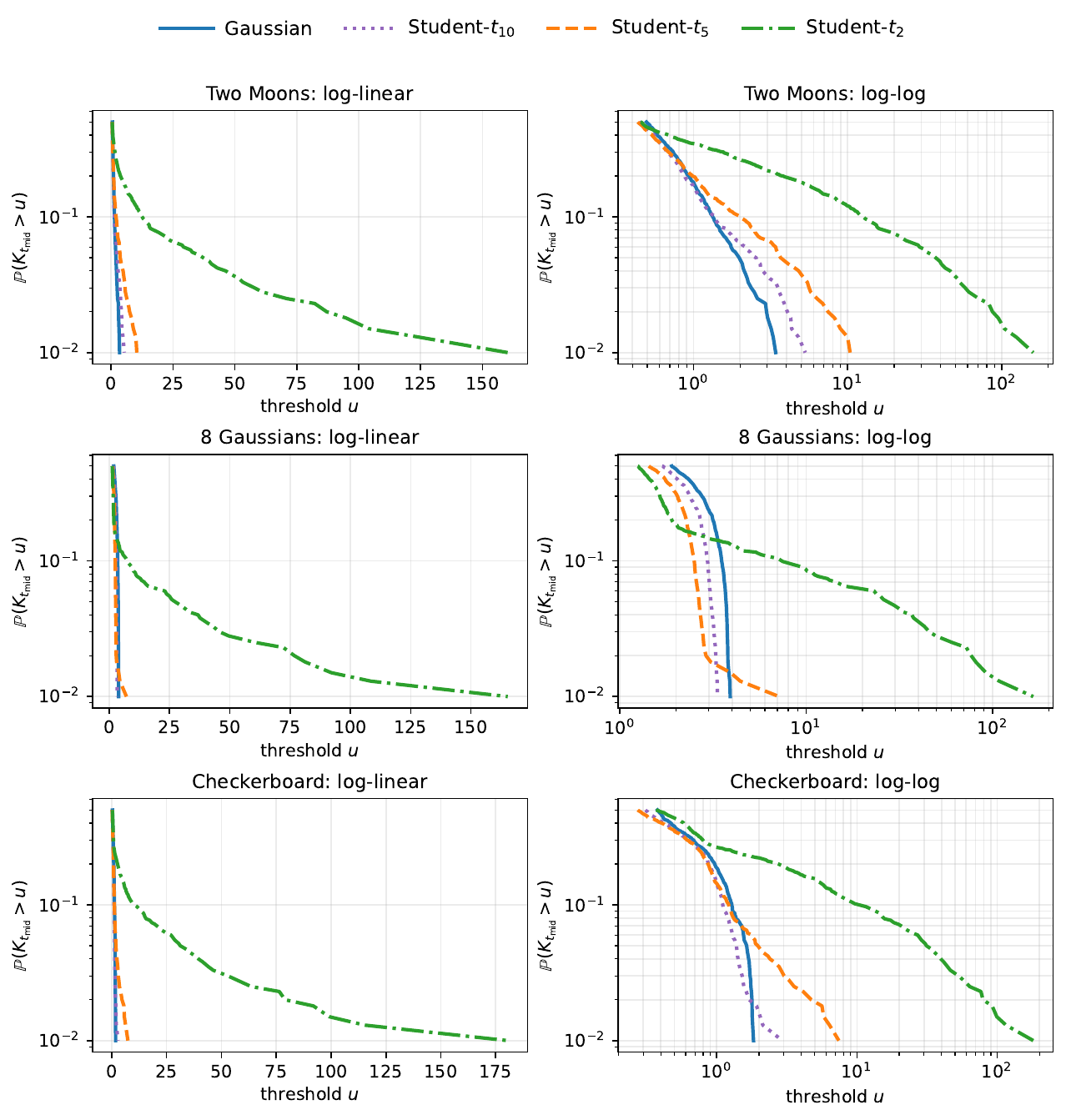}
\caption{
Empirical survival curves for the instantaneous kinetic energy \(K_{t_{\mathrm{mid}}}\). The
source-driven ordering of tail heaviness matches the behavior observed for the integrated energy
\(E_T\).
}
\label{fig:empirical_Kmid_survival_appendix}
\end{figure}

\subsection{Affine Sharpness Experiment}

For the sharpness experiment, we use the affine ODE
\(
    \dot Z_t=AZ_t+b,
\)
with
\(
    A=
    \begin{pmatrix}
    1.2 & 0.35\\
    0.35 & 0.8
    \end{pmatrix},
    \ 
    b=
    \begin{pmatrix}
    0.7\\
    -0.4
    \end{pmatrix}.
\)
The matrix \(A\) is symmetric positive definite. Defining
\(
    Y_t=AZ_t+b,
\)
we obtain
\(
    \dot Y_t=AY_t,
    \ 
    Y_t=e^{tA}(AX_0+b).
\)
Thus,
\(
    E_T
    =
    \int_0^T
    \|e^{tA}(AX_0+b)\|^2\,dt
    =
    (AX_0+b)^\top G_T(AX_0+b),
\)
where
\(
    G_T=
    \int_0^T e^{tA^\top}e^{tA}\,dt.
\)
Since \(A\) is nonsingular and \(G_T\succ0\), this quadratic form is comparable to
\(\|X_0\|^2\) in the tail.

The sharpness experiment settings are shown in Table~\ref{tab:sharpness_numerics_settings}.
The energy integral is approximated with a trapezoidal rule.

\begin{table}[!h]
\centering
\small 
\begin{tabular}{ll}
\toprule
Parameter & Value \\
\midrule
Generated samples \(M\) & \(100000\) \\
Dimension & \(d=2\) \\
Sources & Gaussian, Student-\(t_2\), Student-\(t_5\), Student-\(t_{10}\) \\
ODE & \(\dot Z_t=AZ_t+b\) \\
Integration horizon & \(T=1\) \\
Quadrature steps & \(400\) \\
\bottomrule
\end{tabular}
\caption{Numerical settings for the affine sharpness experiment.}
\label{tab:sharpness_numerics_settings}
\end{table}

\end{document}